\documentclass[10pt,journal,compsoc]{IEEEtran}

\ifCLASSOPTIONcompsoc
  \usepackage[nocompress]{cite}
\else
  \usepackage{cite}
\fi
\ifCLASSINFOpdf
\else
\fi
\usepackage{times}
\usepackage{framed}
\usepackage[utf8]{inputenc} 
\usepackage[T1]{fontenc}    
\usepackage{hyperref}       
\usepackage{url}            
\usepackage{dcolumn,booktabs,tabularx}
\newcolumntype{d}[1]{D{.}{.}{#1}}
\newcolumntype{P}[1]{>{\centering\arraybackslash}p{#1}}
\usepackage{amsfonts}       
\usepackage{nicefrac}       
\usepackage{microtype}      
\usepackage{epsfig}
\usepackage{graphicx}
\usepackage{color, colortbl}
\definecolor{Gray}{gray}{0.95}
\definecolor{Gray2}{gray}{0.75}
\usepackage{pifont}
\newcommand{\cmark}{\ding{51}}%
\newcommand{\xmark}{\ding{55}}%
\usepackage{amsmath, amsthm}
\usepackage{amssymb}
\usepackage{array}
\usepackage{multirow}
\usepackage{subfig} 
\usepackage{siunitx}
\usepackage[ruled,vlined,linesnumbered]{algorithm2e}
\usepackage[inline]{enumitem}
\DeclareMathOperator*{\argmax}{arg\,max}
\DeclareMathOperator*{\argmin}{arg\,min}

\newcommand{\Fcal}{{\cal F}}
\newcommand{\Lcal}{{\cal L}}
\newcommand{\Rcal}{{\cal R}}
\newcommand{\Ecal}{{\cal E}}
\newcommand{\zz}{{\mathbf s}}
\newcommand{\xx}{{\mathbf x}}
\newcommand{\abf}{{\mathbf a}}
\newcommand{\bb}{{\mathbf b}}
\newcommand{\ii}{{\mathbf i}}
\newcommand{\ftheta}{f_{\theta}}
\newcommand{\ZZ}{{\mathbf S}}
\newcommand{\LL}{{\mathbf L}}
\newcommand{\thetab}{{\boldsymbol{\theta}}}
\newcommand{\KK}{{\mathbf W}}
\newcommand{\mm}{{\mathbf m}}
\newcommand{\MM}{{\mathbf M}}
\newcommand{\XX}{{\mathbf X}}
\newcommand{\XXb}{\mathcal{X}_{\text{base}}}
\newcommand{\XXs}{\mathcal{X}_{\text{s}}}
\newcommand{\XXq}{\mathcal{X}_{\text{q}}}

\newcommand{\II}{{\mathbf I}}
\newcommand{\one}{{\mathbf 1}}

\newcommand{\norm}[1]{\lVert#1\rVert}

\makeatletter
\renewcommand\paragraph{\@startsection{paragraph}{4}{\z@}%
  {-3.25ex \@plus -1ex \@minus -0.2ex}%
  {0.01pt}%
  {\raggedsection\normalfont\sectfont\nobreak\size@paragraph}%
}
\makeatother

\newtheorem{Defn}{Definition}
\newtheorem{Prop}{Proposition}
\newtheorem{theorem}{Theorem}
\newtheorem{corollary}{Corollary}
\newtheorem{lemma}{Lemma}
\begin{document}

\title{Transductive Few-Shot Learning: \\ Clustering is All You Need?}

%

\author{Imtiaz Masud Ziko, Malik Boudiaf, Jose Dolz, Eric Granger
        and~Ismail Ben Ayed
\IEEEcompsocitemizethanks{\IEEEcompsocthanksitem Imtiaz Masud Ziko is with Thales - Cortaix, Canada and \'ETS Montreal, Canada.\protect\\
E-mail: imtiaz.ziko@thalesdigital.io
\IEEEcompsocthanksitem Malik Boudiaf, Jose Dolz, Eric Granger and Ismail Ben Ayed are with \'ETS Montreal, Canada.}
}


\IEEEtitleabstractindextext{%
\begin{abstract}
We investigate a general formulation for clustering and transductive few-shot learning, which integrates prototype-based objectives, Laplacian regularization and supervision constraints from a few labeled data points. We propose a concave-convex relaxation of the problem, and derive a computationally efficient block-coordinate bound optimizer, with convergence guarantee. At each iteration, our optimizer computes independent (parallel) updates for each point-to-cluster assignment. Therefore, it could be trivially distributed 
for large-scale clustering and few-shot tasks. Furthermore, we provides a thorough convergence analysis based on point-to-set maps. We report comprehensive clustering and few-shot learning experiments over various data sets, showing that our method yields competitive performances, in term of accuracy and optimization quality, while scaling up to large problems. Using standard training on the base classes, without resorting to complex meta-learning and episodic-training strategies, our approach outperforms state-of-the-art few-shot methods by significant margins, across various models, settings and data sets. Surprisingly, we found that even standard clustering procedures (e.g., K-means), which correspond to particular, non-regularized cases of our general model, already achieve competitive performances in comparison to the state-of-the-art in few-shot learning. These surprising results point to the limitations of the current few-shot benchmarks, and question the viability of a large body of convoluted few-shot learning techniques in the recent literature. Our code is publicly available at \url{https://github.com/imtiazziko/SLK-few-shot}.

\end{abstract}

\begin{IEEEkeywords}
Clustering, few-shot learning, transductive inference, bound optimization, regularization, point-to-set maps.
\end{IEEEkeywords}}

\maketitle
\IEEEdisplaynontitleabstractindextext
\IEEEpeerreviewmaketitle

\IEEEraisesectionheading{\section{Introduction}\label{sec:introduction}}

Deep learning models achieve outstanding performances when trained on large-scale labeled data sets. However, they are seriously challenged when dealing with new classes that were not observed during training, even if a handful of labeled samples becomes available for these new classes. To address this challenge, few-shot learning methods have recently attracted wide interest within the machine learning and computer vision communities. In standard few-shot learning settings, we assume that we have access to training data containing labeled samples from an initial set of classes, commonly referred to as the {\em base} classes. A fully supervised model is first trained on this base data. Then, supervision for a new set of classes that are completely different from those seen in the base training is restricted to just one or a few labeled samples per class. Model generalization is assessed over few-shot {\em tasks}: Each task involves a {\em query} set, which contains unlabeled samples for evaluation, and is supervised by a {\em support} set, which contains a few labeled samples per new class.

The recent and abundant few-shot literature is widely dominated by convoluted meta-learning and episodic-training strategies, which create a set of few-shot tasks (or episodes) during base training, each containing support and query samples, so as to mimic the generalization challenges encountered at testing times. Very popular methods within this framework include \cite{snell2017prototypical,Vinyals2016MatchingNF,Finn2017ModelAgnosticMF,Ravi2017OptimizationAA}. Prototypical networks \cite{snell2017prototypical} use an embedding prototype for each class, and optimize the log-posteriors of the query points, within each episode during base training. Matching networks \cite{Vinyals2016MatchingNF} model the predictions of query samples as linear functions of the support labels, and use memory architectures along with episodic-training strategies. MAML (Model-Agnostic Meta-Learning) \cite{Finn2017ModelAgnosticMF} attempts to enable a model to become easy to fine-tune, and the meta-learner in \cite{Ravi2017OptimizationAA} views optimization as a model for few-shot learning. These widely used methods were followed by a large body of few-shot learning literature based on meta-learning, e.g., \cite{sung2018learning,oreshkin2018tadam, mishra2018a, rusu2018metalearning, liu2018learning,can,ye2020fewshot}, to list a few.

\textbf{Transductive few-shot learning:}
Recently, a large body of works investigated {\em transductive} inference for few-shot tasks, e.g., \cite{liu2018learning,team,can,Dhillon2020A,hu2020empirical,wang2020instance,guo2020attentive,yang2020dpgn,Laplacian,liu2019prototype,liu2020ensemble,malik2020Tim}. 
The transductive setting leverages the statistics of the unlabeled query set and, in general, yields substantial improvements in performances over {\em inductive} inference.
Such boosts in performances are in line with established facts in classical transductive inference for semi-supervised learning \cite{vapnik1999overview,joachim99,z2004learning}, 
well-known to outperform its inductive counterpart on small labeled data sets. While transductive inference uses exactly the same training and testing data as its inductive counterpart\footnote{In the transductive setting, each few-shot task is treated independently of the other tasks. Therefore, the setting does not assume to have access to additional 
unlabeled data. This is different from semi-supervised few-shot learning \cite{ren18fewshotssl}, which uses additional unlabeled data.}, it classifies all the unlabeled query points 
of a few-shot task jointly, rather than one data point at a time as in inductive inference.

As pointed out in \cite{BronskillICML2020}, most of meta-learning methods rely critically on transductive batch normalization (TBN) to achieve competitive performances, for instance, the works in \cite{Finn2017ModelAgnosticMF,GordonICLR2019,ZhenICML2020}, among others. Adopted initially in the widely used MAML \cite{Finn2017ModelAgnosticMF}, TBN performs normalization 
using the statistics of the query set of a given few-shot task, and yields significant increases in performances \cite{BronskillICML2020}. Therefore, due to the popularity of MAML, 
several meta-learning techniques have used TBN. There are also several recent methods that proposed specialized architectures and loss functions for transductive few-shot inference. 
For instance, the meta-learning, graph-based method in \cite{liu2018learning} propagates labels from labeled to unlabeled data points, following label-propagation concepts, which are 
popular in the classical semi-supervised learning literature \cite{z2004learning}. The transductive inference in \cite{Dhillon2020A} encourages confident class predictions at the unlabeled query points of a few-shot task by minimizing the Shannon entropy, in conjunction with the standard cross-entropy loss over the labeled support samples. It is worth noting that the 
general principle of minimizing entropy to leverage unlabeled data is widely used in semi-supervised learning \cite{grandvalet2005semi,miyato2018virtual,berthelot2019mixmatch}. The method in \cite{malik2020Tim} maximized the mutual information between the query samples of a task and their latent labels, and could be viewed as an extension of the entropy loss 
in \cite{Dhillon2020A}. The method in \cite{Laplacian} regularized the inductive predictions of query samples with a pairwise Laplacian term. 

It is important to note that the computational times of transductive-inference techniques are, in general, much heavier than inductive inference. For example, the entropy minimization in \cite{Dhillon2020A} updates all the trainable parameters of a deep network at inference, requiring several orders of magnitude more computations than inductive inference. 
Another example is the transductive label-propagation inference is \cite{liu2018learning}, whose complexity is cubic in the number of unlabeled query points, as it requires matrix inversion.

\subsection{Contributions}

We view transductive few-shot inference as a constrained clustering, and investigate a general multi-term formulation of the problem, which integrates: a prototype-based 
objective (e.g, K-means or K-modes); a pairwise Laplacian regularizer; and supervision constraints from the support set. We propose a concave-convex relaxation of the objective, for which we derive a computationally efficient block-coordinate bound optimizer, with convergence guarantee. At each iteration, our optimizer computes independent (parallel) updates for each point-to-cluster assignment. Therefore, it could be trivially distributed for large-scale clustering and few-shot tasks. Furthermore, we provides a thorough convergence analysis based on point-to-set maps. We report comprehensive clustering and few-shot experiments over various data sets, showing that our method yields competitive performances, both in term of accuracy and optimization quality (i.e., the objectives obtained at convergence), while scaling up to large problems. Without resorting to complex meta-learning and episodic-training schemes, and using a standard training on the base classes, our approach outperforms state-of-the-art few-shot methods by significant margins across different models, settings, and data sets. Furthermore, and surprisingly, we found that even standard clustering procedures (e.g., K-means), which correspond to particular, non-regularized cases of our general model, already achieve competitive performances in comparison to the the state-of-the-art in few-shot learning (even though these are not the best-performing variants of our general model). These surprising results point to the 
limitations of current few-shot benchmarks, and question the viability of a large body of convoluted few-shot learning techniques in the recent literature.


\subsection{Related works}

Our work integrates several powerful ideas in clustering: pairwise graph-Laplacian regularization, cluster-representative objectives and density-mode finding. Cluster-representative objectives, such as the standard K-means, are very popular because they provide prototype representations of the clusters, explicit cluster-assignment variables and straightforward out-of-sample extensions. Furthermore, typically, they are amenable to optimizers that scale up to large and high-dimensional data sets. Pairwise Laplacian regularization encourages data points that are neighbors in the feature space to have consistent latent assignments, and is widely used in the contexts of spectral graph clustering \cite{ShiMalik2000,von2007tutorial,shaham2018spectralnet,Tang2019KernelCK} and semi-supervised learning \cite{belkin2006manifold,weston2012deep,IscenTAC19,TangPDASB18,Chapelle2010,z2004learning}. Unlike standard prototype-based methods, Laplacian regularization can separate non-convex (or manifold-structured) clusters. However, despite their complementary benefits, graph-Laplacian and prototype-based terms are, in general, used separately in the clustering literature, except a very few exceptions \cite{ziko2018scalable,WangCarreira-Perpinan2014,Tang2019KernelCK}. This might be due to the significant differences in the optimization algorithms used to tackle these two different types of objectives. 

On the one hand, prototype-based objectives are commonly tackled with block-coordinate descent procedures, which alternate optimization 
with respect to the cluster prototypes and discrete assignment variables. It is worth noting that optimizing prototype-based objectives over discrete 
variables is NP-hard\footnote{For instance, a proof of the NP-hardness of K-means could be found in \cite{AloiseDHP09}.}.
On the other hand, optimizing the Laplacian term over discrete variables, which is also NP-hard \cite{Tian-AAAI}, is commonly tackled by 
relaxing the integer constraints. For instance, in the context of graph clustering, spectral relaxation \cite{von2007tutorial, ShiMalik2000} is widely used for optimizing the Laplacian term subject to balancing constraints. The problem is expressed in the form of a generalized Rayleigh quotient, which yields a closed-form solution in terms of the $K$ largest eigenvectors of the affinity matrix ($K$ is the number of clusters). However, spectral relaxation is not applicable to joint objective integrating K-prototypes and Laplacian regularization, as in our formulation. Furthermore, spectral relaxation has high computational and memory 
load for large data sets: For $N$ data points, one has to store an $N \times N$ affinity matrix and compute its eigenvalue decomposition. 
The complexity is cubic with respect to $N$ for a straightforward implementation and, to our knowledge, super-quadratic for fast implementations \cite{Tian-AAAI}. 
Another approach is to replaces the integer constraint with a probability-simplex constraint, which results in a convex relaxation of the Laplacian term \cite{WangCarreira-Perpinan2014}. Such a direct convex relaxation could be integrated with prototype-based objectives \cite{WangCarreira-Perpinan2014}. However, unfortunately, it requires solving for $N \times K$ variables jointly. Furthermore, it requires additional projections onto the $K$-dimensional simplex, with a quadratic complexity with respect to $K$. Therefore, as we will see in our experiments, the relaxation in \cite{WangCarreira-Perpinan2014} does not scale up for large-scale problems (i.e., when $N$ and $K$ are large). 

Laplacian regularization is also widely used in semi-supervised learning \cite{belkin2006manifold,weston2012deep,IscenTAC19,TangPDASB18,ZhuICML2003Guassianharmonics,LiangLPIJCAI18}. 
This includes, most notably, label-propagation methods \cite{ZhuICML2003Guassianharmonics,Chapelle2010,z2004learning}, which spread soft labels from a few labeled samples to unlabeled ones. Popular label-propagation algorithms \cite{ZhuICML2003Guassianharmonics,z2004learning} could be unified under a general convex-relaxation framework \cite{Chapelle2010}, which minimizes a quadratic objective function, integrating 
a Laplacian regularization term and a supervised-learning term defined over the labeled data points, along with additional terms to avoid degenerate solutions. While the relaxation is convex, solving the ensuing necessary condition for an optimum has high computational and memory load as it requires an $N \times N$ matrix inversion. In general, such matrix inversions in the standard 
label-propagation algorithms take O($N^3$) time and O($N^2$) memory \cite{LiangLPIJCAI18}. This impedes the use of label propagation for large-scale data sets.    

Our concave-convex relaxation computes independent (parallel) pointwise updates during each iteration, and has important computational and memory advantages over the above-discussed methods.
It does not require eigenvalue decomposition/matrix inversion (unlike spectral relaxation and label propagation), neither does it perform expensive projection steps and Lagrangian-dual inner iterates for the simplex constraints of each point (unlike convex relaxation). In fact, investigating the scalability of spectral relaxation for large-scale problems is an active research subject, for both laplacian regularization and clustering in general \cite{shaham2018spectralnet,LiangLPIJCAI18,Tian-AAAI,Vladymyrov-2016}. For instance, the studies in \cite{shaham2018spectralnet,Tian-AAAI} investigated deep-learning approaches to spectral clustering, so as to ease the scalability issues for large data sets, and the authors of \cite{Vladymyrov-2016} examined the variational Nystr{\"{o}}m method for large-scale spectral problems, among many other efforts on the subject. 
Also, there is an active line of research works to mitigate the scalability issue of label propagation \cite{LiangLPIJCAI18}, e.g., via graph approximations or stochastic optimization methods. In fact, such scalability issues are being actively investigated even for basic clustering procedures such as K-means \cite{gong2015web,newling-fleuret-2016b}.

Finally, we point to the preliminary clustering results of this work, which we published at the NeurIPS conference \cite{ziko2018scalable}. 
This journal extension expands significantly on \cite{ziko2018scalable}, and generalizes the work in many different ways, including: (i) a 
generalized and constrained version of our formulation, along with its adaptation to transductive few-shot learning; (ii) a new convergence
analysis based on point-to-set maps and Cauchy sequences; (iii) a completely new set of comprehensive experiments and comparisons in the context of 
few-shot classification, along with surprising findings; and (iv) new discussions of the recent few-shot learning literature.

\section{Proposed Formulation}

    We advocate Laplacian K-prototypes for clustering and transductive few-shot inference, and propose a general bound- and relaxation-based solution of the problem, which can be used to regularize several prototype-based objective (e.g, K-means, K-median or K-modes). Importantly, our concave-convex relaxation yields a parallel-structure algorithm, with convergence guarantee, and can handle large-scale data sets. First, we discuss a general formulation for clustering. Then, we describe an adaptation to transductive few-shot learning, which has recently attracted significant interest. In this case, our model takes the form of a constrained clustering for a given few-shot task. 
    
    Let $\mathcal{X} = \{ \xx_p \in \mathbb{R}^d, p = 1 , \dots, N\}$ denotes a set of feature inputs, which need to be assigned to $K$ different clusters $\mathcal{S}_k, \, k=1, \dots, K$. For each point $\xx_p$, we define a latent one-hot assignment vector $\zz_p = [s_{p,1}, \dots, s_{p,K}]^t \in\{0,1\}^K$, which is constrained to be within the $K$-dimensional simplex: $s_{p,k} = 1$ if $\xx_p$ belongs to cluster $\mathcal{S}_k$ and $s_{p,k} = 0$ otherwise. Superscript $t$ denotes the transpose operator. Let $\ZZ$ denote the $N \times K$ matrix whose rows are formed by assignment variables $\zz_p$, $p = 1, \dots, N$. Finally, each cluster $\mathcal{S}_k$ is represented by a prototype variable $\mm_k \in \mathbb R^d$ in the feature space. Let $\MM$ be  the $K \times d$ prototype matrix. We consider the following general Laplacian K-prototypes model for clustering:
    \begin{align}
        \label{eq:lkmode0}
        \min_{\ZZ, \MM}& \quad {\Ecal}(\ZZ, \MM) = {\Fcal(\ZZ, \MM)} + \frac{\lambda}{2} {\Lcal(\ZZ)} \nonumber \\
        \text{s.t.}& \quad \one^t\zz_p = 1; \,  \zz_p \in \{0,1\}^{K} \, \forall p  
    \end{align}
    General model \eqref{eq:lkmode0} integrates, within a single objective, the advantages of several powerful and well-known ideas in clustering and SSL, which 
    makes it very convenient:
    \begin{itemize}
        \item \textbf{Prototype-based term}: $\Fcal(\ZZ, \MM)$ is a standard prototype-based objective such as K-means or K-modes; see Table \ref{tab:a_p} for detailed expressions. It identifies cluster prototypes (or representatives) $\mm_k$, for instance, modes or means. It is worth noting that the K-modes term in Table \ref{tab:a_p} is closely related to density mode estimation methods, for instance, the very popular mean-shift algorithm \cite{ComaniciuMeer2002}. The value of $\mm_k$ globally optimizing this term for a given fixed cluster $\mathcal{S}_k$ is the mode of the kernel density of feature points within the cluster \cite{Tang2019KernelCK}. Therefore, the K-modes term, as in \cite{Carreira-PerpinanWang2013}, can be viewed as an energy-based formulation of mean-shift algorithms with a fixed number of clusters \cite{Tang2019KernelCK}. Cluster modes are valid data points in the input set and, therefore, convenient representations for manifold-structured, high-dimensional inputs such as images, where simple parametric prototypes such as the means, as in K-means, may not be good representatives.
        
        \item \textbf{Laplacian term}: 
            $\Lcal(\ZZ)$ is the well-known graph Laplacian regularizer:
            \begin{eqnarray}
            \label{eq:lap}
             {\Lcal(\ZZ)} & = & \sum_{p,q} w_{\mathcal{L}}(\xx_p, \xx_q) \|\zz_p - \zz_q\|^2
             \end{eqnarray}
            This term can be equivalently written as $ \mathrm{tr}(\ZZ^t \LL \ZZ)$, with $\LL$ the Laplacian matrix corresponding to affinity matrix $\KK_{\mathcal{L}} = [w_{\mathcal{L}}(\xx_p, \xx_q)]$. Laplacian regularization encourages nearby data points to have similar latent assignments, and is widely used in spectral clustering \cite{ShiMalik2000,von2007tutorial,shaham2018spectralnet,Tang2019KernelCK}, as well as in SSL \cite{belkin2006manifold,weston2012deep,IscenTAC19,TangPDASB18}. 
            
            Model \eqref{eq:lkmode0} is flexible: On the one hand, Laplacian term $ {\Lcal(\ZZ)}$ enables to handle non-convex (or manifold-structured) clusters, unlike standard prototype-based clustering techniques such as K-means. On the other hand, and unlike spectral clustering methods, term ${\Fcal(\ZZ, \MM)}$ yields prototype representations of the clusters.
            Furthermore, as we will see later in  the few-shot setting and in our experiments, model \eqref{eq:lkmode0} enables to easily and effectively embed supervision information from a few labeled examples in the prototype terms. It is not clear how to embed effectively such supervision constraints in spectral clustering, as the latter takes the form of a generalized Rayleigh quotient optimization. 
    \end{itemize} 
    
    \subsection{Adaptation to transductive few-shot inference}
        
        In standard few-shot learning settings, we assume that we have access to a training set $\XXb$ containing labeled data points from an initial set of classes, referred to as the {\em base} classes. A fully supervised deep-network model is first trained on this base data, yielding embedding function $\xx_p=\ftheta(\ii_p)$ for a given data input $\ii_p$, with $\theta$ denoting the base-training parameters. Then, supervision for a new set of classes, which are completely different from those seen in the base training, is restricted to just one or a few labeled samples per class. Model generalization is assessed over few-shot {\em tasks}: Each task involves a {\em query set}, which contains unlabeled samples for evaluation, and is supervised by a {\em support set}, which contains a few labeled samples per new class. For a given few-shot task, let $\XXs=\bigcup_{k=1}^{K} \XXs^k$ denote the labeled support set, with $K$ the number of classes of the task. Each new class $k$ has $|\XXs^k|$ labeled examples, for instance, $|\XXs^k|=1$ for 1-shot and $|\XXs^k|=5$ for 5-shot. The objective is, therefore, to accurately classify the unlabeled samples of query set $\XXq=\bigcup_{k=1}^{K} \XXq^k$ into one of these $K$ test classes. This setting is referred to as the $|\XXs^k|$-shot $K$-way few-shot learning. 
        For each few-shot task with concatenated features ${\mathcal{X} =\{ \xx_p \in \mathbb{R}^D, p = 1 , \dots, N\} = \XXs \cup \XXq}$, we pose the problem as 
        optimizing a Laplacian K-prototypes objective, subject to supervision constraints from the support set: 
        \begin{eqnarray}
        \label{eq:lkmode1}
         \min_{\ZZ, \MM} \ {\Ecal}(\ZZ, \MM)  &=& {\Fcal(\ZZ, \MM)} + \frac{\lambda}{2} {\Lcal(\ZZ)} \quad \mbox{s.t.}  \nonumber \\
         s_{p,k} &=& 1~\mbox{if}~ \xx_p \in \XXs^k  \nonumber \\
         \one^t\zz_p &=& 1, \; \zz_p \in \{0,1\}^{K} \; \forall p 
        \end{eqnarray}

        \begin{table*}[t]
        \caption{Auxiliary functions of several prototype-based clustering objectives. ${\mathbf S}_k = [s_{1,k}, \dots, s_{N,k}]^t \in\{0,1\}^N$ is the binary indicator vector of cluster $\mathcal{S}_k$. When used for variables ${\mathbf S}_k$ and ${\mathbf m}_k$, superscript $i$ means the current solution at outer iteration indexed by $i$. $\XX$ is the $N \times d$ 
        matrix whose rows are formed by feature inputs $\xx_p$.}
        \label{tab:a_p}
        \begin{center}
        {\renewcommand{\arraystretch}{1.2}
        \begin{tabular}{P{2.1cm}P{3.5cm}P{4.2cm}P{5.5cm}}
        \toprule
        \textbf{Clustering} & $\Fcal(\ZZ, \MM)$ &$\abf_p^i= [a_{p,1}^i, \dots,a_{p,K}^i]^t$ & \textbf{Where}\\
        \toprule
        K-means & $\sum_p\sum_ks_{p,k}(\xx_p - \mm_k)^2$ &$(\xx_p - \mm_k^i)^2$&$\mm_k^i = \frac{\XX^t {\mathbf S}_k^i}{\one^t {\mathbf S}_k^i}$ \\

        \midrule
        K-modes & $ - \sum_p\sum_k s_{p,k}w_{\mathcal{F}}(\xx_p,\mm_k)$ & $ - w_{\mathcal{F}}(\xx_p,\mm_k^i)$ & $w_{\mathcal{F}}$ is kernel-based affinity \newline $\mm_k^i = \argmax_{\mm} \sum_{p} s_{p,k}^i w_{\mathcal{F}}(\xx_p,\mm)$\\
        \bottomrule
        \end{tabular}
        }
        \end{center}
        \end{table*}
        
\section{Concave-convex relaxation} \label{sec:concave-convex}
    It is easy to verify that, for binary simplex variables, the Laplacian term $\Lcal(\ZZ)$ can be written as follows: 
    \begin{eqnarray}
    \label{tight-relaxation}
    \Lcal(\ZZ) &&= \mathrm{tr}(\ZZ^t \LL \ZZ) \nonumber  \\
               &&= 2 \left(\sum_{p} \zz_p^t\zz_p d_{p} - \sum_{p,q} w_{\mathcal{L}}(\xx_p, \xx_q) \zz_p^t \zz_q \right) \nonumber \\
               &&= 2 \left(\sum_{p} d_{p} - \sum_{p,q} w_{\mathcal{L}}(\xx_p, \xx_q) \zz_{p}^{t} \zz_q \right)
    \end{eqnarray} 
    where the last equality is valid only for binary (integer) variables, and $d_p$ denotes the degree of data point $p$:  
    \[d_p = \sum_q w_{\mathcal{L}}(\xx_p,\xx_q)\]
    When we replace the integer constraints $\zz_p \in \{0,1\}$ by $\zz_p \in [0,1]$, our relaxation becomes different from direct convex relaxations of the Laplacian \cite{WangCarreira-Perpinan2014}, which optimizes $\mathrm{tr}(\ZZ^t \LL \ZZ)$ subject to probabilistic simplex constraints. For relaxed variables, $\mathrm{tr}(\ZZ^t \LL \ZZ)$ is a convex function because the Laplacian is always positive semi-definite (PSD). 
    In contrast, our relaxation of the Laplacian term is concave for PSD affinity matrix $\KK_{\mathcal{L}} = [w_{\mathcal{L}}(\xx_p, \xx_q)]$. As we will see later, concavity yields a scalable (parallel) algorithm for large $N$, which computes independent updates for assignment variables $\zz_p$. Our updates can be trivially distributed, and do not require storing a full $N \times N$ affinity matrix. These are important computational and memory advantages over direct convex relaxations of the Laplacian \cite{WangCarreira-Perpinan2014}, which require solving for $N \times K$ variables all together as well as expensive simplex projections, and over common spectral relaxations \cite{von2007tutorial}, which require storing a full affinity matrix and computing its eigenvalue decomposition.
    
    We propose the following concave-convex relaxation of the objective in \eqref{eq:lkmode0} by relaxing the integer constraints and adding a convex negative-entropy barrier for constraints $\zz_p \geq 0$:
    \begin{equation}
        \label{Concave-Convex-relaxation}
        {\Rcal}(\ZZ,  \MM) = \Ecal(\ZZ, \MM) + \mathrm{tr}(\ZZ \log \ZZ^{t})
    \end{equation}
    ${\cal R}(\ZZ, \MM)$ is minimized over each $\zz_p \in \nabla_K$, with $\nabla_K$ denoting the $K$-dimensional probability simplex:
    \[\nabla_K = \{\zz \in [0, 1]^K \; | \; {\mathbf 1}^t \zz = 1 \}\] 
    It is easy to check that, at the vertices of the simplex, our relaxation in \eqref{Concave-Convex-relaxation} is equivalent to the initial discrete objective in Eq. \eqref{eq:lkmode0}. 
    The last term we introduced in Eq. \eqref{Concave-Convex-relaxation} is a convex negative-entropy barrier function, which completely avoids expensive projection steps and Lagrangian-dual inner iterations for the simplex constraints of each point. First, this entropy barrier restricts the domain of each $\zz_p$ to non-negative values, which avoids extra dual variables for constraints $\zz_p \geq 0$. Second, the presence of such a barrier function yields closed-form updates for the dual variables of constraints $\mathbf{1}^t\zz_p = 1$. In fact, entropy-like barriers are commonly used in Bregman-proximal optimization \cite{Yuan2017}, and have well-known computational and memory advantages when dealing with the challenging simplex constraints \cite{Yuan2017}. Surprisingly, to our knowledge, they are not common in the context of clustering. In machine learning, such entropic barriers appear frequently in the context of conditional random fields (CRFs) \cite{Krahenbuhl2011,krahenbuhl2013parameter}, but are not motivated from an optimization perspective; they result from standard probabilistic and mean-field approximations of CRFs \cite{Krahenbuhl2011}.

\section{Optimization}
    \label{sec:bound-optimization}   
    
    In this section, we derive a block-coordinate bound  optimization algorithm for tackling relaxation \eqref{Concave-Convex-relaxation}. 
    Our algorithm alternatively computes parallel updates of assignment variables $\zz_p$ ($\zz$-updates), for each point $p$ and independently of the other points, along
    with prototype updates ($\MM$-updates) at each outer iteration, with guaranteed convergence. 
    As we will see in our experiments, our bound optimizer yields consistently lower values of objective ${\cal E}$ at convergence than the proximal algorithm 
    in \cite{WangCarreira-Perpinan2014}, while being highly scalable to large-scale and high-dimensional problems. Our SLK (Scalable Laplacian K-prototypes) algorithm alternates 
    the following two steps until convergence:
    \begin{itemize}
        \item Optimizing $\cal{R}(\ZZ, \MM)$ w.r.t $\ZZ$, with $\MM$ fixed (\autoref{subsec:assignment_udpate}).
        \item Optimizing $\cal{R}(\ZZ, \MM)$ w.r.t $\MM$, with $\ZZ$ fixed (\autoref{sec:mode}).
    \end{itemize}

    \subsection{Bound optimization for the $\ZZ$-updates} \label{subsec:assignment_udpate}
        At outer iteration $i$, we keep $\MM = \MM^i$ fixed and minimize $\mathcal{R}^i(\ZZ)=\mathcal{R}(\ZZ, \MM^i)$ with respect to $\ZZ$. To tackle efficiently the minimization of our relaxation $\mathcal{R}^i$, we minimize a series of auxiliary functions ${\cal A}^{i,n}$, indexed by inner counter $n \in \mathbb N$. As an auxiliary function, ${\cal A}^{i,n}$ is an upper bound on ${\cal R}^i$, is tight at the current inner solution ${\ZZ}^{i,n}$ and is chosen to make optimization easier. Formally:
        \begin{Defn}
            At outer iteration $i$, and inner iteration $n$, ${\cal A}^{i,n}(\ZZ) $ is an auxiliary function of $\mathcal{R}^i(\ZZ)$ at current solution $\ZZ^{i, n}$ if it satisfies:
            \begin{subequations}
                \begin{align}
                    \mathcal{R}^i(\ZZ) \, &\leq \, {\cal A}^{i, n}(\ZZ), \, \forall~\ZZ \label{general-second-aux} \\
                    \mathcal{R}^i(\ZZ^{i,n}) &= {\cal A}^{i, n}(\ZZ^{i,n}) \label{general-third-aux} 
                \end{align}
                \label{Eq:Auxiliary_function_conditions}
            \end{subequations}
        \end{Defn}
        In general, bound optimizers update current solution $\ZZ^{i,n}$ to the optimum of the auxiliary function: 
        $\ZZ^{i,n+1} = \arg \min_{\ZZ}~{\cal A}^{i, n}(\ZZ)$. This guarantees that, at each inner iteration, the original objective function does not increase: 
        \begin{align}
            {\cal R}^i(\ZZ^{i, n+1}) \leq {\cal A}^{i, n}(\ZZ^{i, n+1}) \leq {\cal A}^{i, n}(\ZZ^{i, n}) = {\cal R}^i(\ZZ^{i, n})
        \end{align}
        Bound optimizers can be very effective as they transform difficult problems into easier ones \cite{Zhang2007}. Examples of well-known bound optimizers include the concave-convex procedure (CCCP) \cite{Yuille2001}, expectation maximization (EM) algorithms and submodular-supermodular procedures (SSP) \cite{Narasimhan2005}, among others. Furthermore, bound optimizers are not restricted to differentiable functions\footnote{Our objective is not differentiable with respect to the modes as each of these is defined as the maximum of a function of the assignment variables.}, neither do they depend on optimization parameters such as step sizes. 
        
        \begin{Prop}
            At outer iteration $i$, inner iteration $n$, and given the current  solution $\ZZ^{i,n} = [s_{p,k}^{i,n}]$, we have the following auxiliary function (up to an additive constant) for concave-convex relaxation $\mathcal{R}^{i}(\ZZ)$ and psd\footnote{We can consider $\KK_{\mathcal{L}}$ to be psd without loss of generality. When $\KK_{\mathcal{L}}$ is not psd, we can use a diagonal shift for the affinity matrix, i.e., we replace $\KK_{\mathcal{L}}$ by $\KK_{\mathcal{L}}+\delta \II_N$. Clearly, $\KK_{\mathcal{L}}+\delta \II_N$ is psd for sufficiently large $\delta$. For integer variables, this change does not alter the structure of the minimum of discrete function ${\cal E}$.} affinity matrix $\KK_{\mathcal{L}}=[w_{\mathcal{L}}(\xx_p,\xx_q)]_{p,q}$:
            \begin{equation}
                \label{Aux-function}
                {\cal A}^{i, n}(\ZZ) =  \sum_{p =1}^{N} \zz_p^t \left (\log (\zz_p) - {\abf}_p^i - \lambda~{\bb}_p^{i, n} \right )
            \end{equation}
            where ${\abf}_p^i$ and ${\bb}_p^{i, n}$ are the following $K$-dimensional vectors:
            \begin{subequations} 
                \begin{align}
                    {\abf}_p^i &=  [a_{p,1}^i, \dots,a_{p,K}^i]^t,  \, \mbox{\em with}~ a_{p,k}^i~\mbox{given in Table \ref{tab:a_p}} \label{general-second} \\
                    {\bb}_p^{i,n} &= [b_{p,1}^{i,n}, \dots,b_{p,K}^{i,n}]^t,  \, \mbox{\em with} \, \, b_{p,k}^{i,n} =   \sum_{q} w_\mathcal{L}(\xx_p, \xx_q) s_{q,k}^{i,n} \label{general-third} 
                \end{align}
            \end{subequations}
        \end{Prop}
        \begin{proof}
        See Appendix A.
        \end{proof}
        
        Notice that the bound in Eq. \eqref{Aux-function} is the sum of independent functions, each corresponding to a point $p$. As a result, both the bound and simplex constraints $\zz_p \in \nabla_K$ are separable over assignment variables $\zz_p$. Therefore, inner iteration $n$ corresponds to minimizing the auxiliary function $\mathcal{A}^{i,n}$ by minimizing independently each term in the sum over $\zz_p$, subject to the simplex constraint, while guaranteeing convergence: 
        \begin{equation}
            \label{AUX-Form-each-variable}
            \min_{\zz_p \in \nabla_K} \zz_p^t (\log (\zz_p) - {\abf}_p^i - \lambda{\bb}_p^{i,n}), \, \forall p
        \end{equation}
        Note that, for each $p$, negative entropy $\zz_p^t \log \zz_p$ restricts $\zz_p$ to be non-negative, which removes the need for handling explicitly constraints $\zz_p \geq 0$. This term is convex and, therefore, the problem in \eqref{AUX-Form-each-variable} is convex: 
        The objective is convex (sum of linear and convex functions) and constraint $\zz_p \in \nabla_K$ is affine. Hence, one can minimize this constrained convex problem for each $p$ by solving the Karush-Kuhn-Tucker (KKT) conditions\footnote{Note that strong duality holds since the objectives are convex and the simplex constraints are affine. This means that the solutions of the (KKT) conditions minimize the auxiliary function.}. The KKT conditions yield a closed-form solution for both primal variables $\zz_p$ and the dual variables (Lagrange multipliers) corresponding to simplex constraints ${\mathbf 1}^t \zz_p = 1$.
        Each closed-form update, which globally optimizes \eqref{AUX-Form-each-variable} and is within the simplex, is given by:
        \begin{equation}
            \label{eq:s_inner_update}
            \zz_{p}^{i, n+1} = \frac{\exp ({\abf}_p^i+ \lambda{\bb}_p^{i, n}) }{{\mathbf 1}^t \exp ({\abf}_p^i+ \lambda{\bb}_p^{i, n})} \, \, \forall \, p 
        \end{equation}
        
        Inner updates \eqref{eq:s_inner_update} are repeated for each point $p$ until convergence. Letting $\zz_{p}^{i, *}$ denote the stationary solution obtained for 
        point $p$ at convergence, the outer update then reads:
        \begin{align}
            \label{eq:prot_mode_update}
            \zz^i_p \leftarrow \zz^{i, *}_p
        \end{align}
        The complexity of each $\ZZ$ inner iteration is $\mathcal{O}(N \rho K)$, with $\rho$ the neighborhood size for the affinity matrix. Typically, we use sparse matrices ($\rho <<N$). Note that the complexity becomes $\mathcal{O}(N^2K)$ for dense matrices, i.e., the case where all the affinities are non-zero. However, the update of each $\zz_p$ can be done for 
        point $p$ independently of the other points, which enables parallel implementations.  
        \begin{algorithm}
              \DontPrintSemicolon
              \SetAlgoLined
            \SetKwInOut{Input}{Input}\SetKwInOut{Output}{Output}
            \Input{$\XX$, Initial prototypes $\MM^0$}
            \BlankLine
            \Output{$\ZZ$ and $\MM$ }
            \BlankLine
            \BlankLine
            $i\leftarrow 1$\tcp*[l]{outer iteration index}
            $\MM^i\leftarrow \MM^0$\;
            \Repeat{convergence}{
            \BlankLine
            \tcp{$\ZZ$-updates}
              \ForEach{$\mathbf{x}_p$}{
                    Compute $\mathbf{a}^i_p$ from \eqref{general-second}\;
                    Initialize $\zz^{i,0}_{p} \leftarrow \frac{\exp\{\mathbf{a}^i_p\}}{\mathbf{1}^t\exp\{\mathbf{a}^i_p\}}$\;
                    Obtain $\zz^{i,*}_{p}$ by repeating \eqref{eq:s_inner_update} until convergence\;
                    Update $\zz^i_{p}$ using \eqref{eq:prot_mode_update} \;
             }
            \BlankLine
            \tcp{$\MM$-updates}
            \ForEach{$k$}{
                	\uIf{SLK-MS}{
                	    Obtain $\mm^{i,*}_k$ by repeating \eqref{eq:fixed_point_iterates} until convergence\;
                        Update $\mm^i_k$ using \eqref{eq:prot_mode_update};
                    }
                    \uIf{SLK-Means}{
                            Update $\mm^i_k$ using \eqref{eq:prot_mean_update}; 
                	}
                }
            $\MM^{i+1} \leftarrow [\mm^i_1; \dots; \mm^i_K]$ \;
            $i\leftarrow i+1$
            }
            $\ZZ \leftarrow [\zz^i_1; \dots; \zz^i_N]$ \;
            $\MM \leftarrow \MM^i$ \;
              \KwRet $\ZZ$, $\MM$
            \caption{SLK algorithm}
            \label{algo:slk}
        \end{algorithm}

    \subsection{Prototype updates}
        \label{sec:mode}
        
        In this section, we give the update rules for prototype matrix $\MM$. Therefore, we consider $\ZZ=\ZZ^i$ fixed to its value at outer iteration $i$. These updates come in two different forms, one is closed-form and the other iterative (based on convergent fixed-point iterations), depending on the chosen prototype-based term $\cal F(\ZZ, \MM)$. In what follows, we detail the prototype updates for the K-modes and K-means objectives.
        
        \subsubsection{SLK-MS}
        
        In this case, the goal is to find $\mm^i_k$ that minimizes the K-modes objective, with the assignment variables fixed, i.e.,
        \begin{align}
            \label{eq:rbf_objective}
            \mm^i_k = \argmin_{\mm} - \sum_p \sum_k s^i_{p, k} w_{\mathcal{F}}(\xx_p, \mm)
        \end{align}
        For the K-modes term, we use the well-known RBF kernel, defined as follows:
        \begin{align}\label{eq:rbf_kernel}
            w_{\mathcal{F}}(\xx_p, \mm) = \exp \left (-\frac{\norm{\xx_p - \mm}^2}{2\sigma^2} \right )
        \end{align}
        with $\sigma > 0$ denoting a kernel width parameter. Setting the gradient of Eq. (\ref{eq:rbf_objective}) w.r.t $\mm$ to ${\mathbf 0}$ yields the following 
        condition for a stationary point of the K-modes objective in \eqref{eq:rbf_objective}:
        \begin{align}
            \label{eq:zero_eq}& \mm - g^i_{k}(\mm) = {\mathbf 0} \\
            \text{with}~&g^i_{k} (\mm) = \frac{\sum_p s^i_{p,k}w_{\mathcal{F}}(\xx_p, \mm) \xx_p}{\sum_p s^i_{p,k}w_{\mathcal{F}}(\xx_p, \mm)}
        \end{align}
        The following proposition describes how to find the unique solution of Eq. (\ref{eq:zero_eq}), via convergent fixed-point iterates:
        \begin{Prop}{ \label{prop:prototype_update}
            For every outer iteration $i$, and every class $k$, the following fixed-point updates:
            \begin{align}
                \label{eq:fixed_point_iterates}
                \mm_k^{i, n+1} = g^i_{k} (\mm_k^{i, n})
            \end{align}
            yield a Cauchy sequence $\{ \mm_k^{i, n}\}_{n \in \mathbb N}$. As such, this sequence converges to a unique value $\mm_k^{i, *} = \lim_{n \rightarrow \infty} \mm_k^{i, n}$. Furthermore, $\mm_k^{i, *}$ is the unique solution of the stationary-point condition in Eq. (\ref{eq:zero_eq})}.
        \end{Prop}
        The proof of Proposition \ref{prop:prototype_update} is provided in Appendix B. K-modes prototypes' updates therefore require inner loop iterations given by Proposition \ref{prop:prototype_update}. Once convergence is reached, the prototypes can be updated as:
        \begin{align}
            \label{eq:prot_mode_update}
            \mm^i_k \leftarrow \mm_k^{i, *}
        \end{align}
        
        \subsubsection{SLK-Means}
        
        In this case, for each class $k$, we compute $\mm_k$ such that:
        \begin{align}
            \mm^i_k = \argmin_{\mm} \sum_p \sum_k s^i_{p, k} \norm{\xx_p - \mm}^2
        \end{align}
        This problem is convex and can be solved in closed-form, yielding the class-wise means as prototypes, as in the standard K-means procedure:
        
        \begin{align}
            \label{eq:prot_mean_update}
            \mm^i_k \leftarrow \frac{\XX^t {\mathbf S}^i_k}{\one^t {\mathbf S}^i_k}
        \end{align}
        where ${\mathbf S}_k = [s^i_{1,k}, \dots, s^i_{N,k}]^t \in\{0,1\}^N$ and $\XX$ is the $N \times d$ matrix whose rows are formed by feature inputs $\xx_p$.
        
        The pseudo-code for our full Scalable Laplacian K-modes (SLK) method is provided in Algorithm \ref{algo:slk}.

\section{Convergence study}

    Putting together the $\ZZ$- and $\MM$-update steps, one could notice that sequence $\{\mathcal{R}(\ZZ^i, \MM^i)\}_{i \in \mathbb N}$ is non-increasing:
    \begin{align}
        \mathcal{R}(\ZZ^{i+1}, \MM^{i+1}) \underbrace{\leq}_{\MM\text{-update}} \mathcal{R}(\ZZ^{i+1}, \MM^{i}) \underbrace{\leq}_{\ZZ\text{-update}} \mathcal{R}(\ZZ^{i}, \MM^{i})
    \end{align}
    Furthermore, this sequence is bounded from below:
    \begin{align}
        \mathcal{R}(\ZZ, \MM) = \underbrace{\mathcal{F}(\ZZ, \MM)}_{\resizebox{0.3\hsize}{!}{$\begin{cases} \geq 0 & \text{SLK-Means} \\ \geq -N & \text{SLK-MS} \end{cases}$}} + \frac{\lambda}{2} \underbrace{\mathcal{L}(\ZZ)}_{\geq 0} + \underbrace{\mathrm{tr}(\ZZ \log \ZZ^t)}_{-N \log K}
    \end{align}
    This makes of $\{\mathcal{R}(\ZZ^i, \MM^i)\}_{i \in \mathbb N}$ a converging sequence, regardless of initialization $\{\ZZ^0, \MM^0\}$. 
    However, without additional analysis, this does neither imply that sequence $\{\ZZ^i, \MM^i\}_{i \in \mathbb N}$ itself converges, nor does it characterize its limit points. 
    
    In what follows, we analyze our proposed SLK algorithm through the lens of Zangwill’s global convergence theory, which provides a simple but general framework to study the convergence of iterative algorithms. Note that this theory was already used to prove the convergence of the concave-convex \cite{sriperumbudur2009convergence} and EM/GEM procedures \cite{wu1983convergence}, both of which could be viewed as special cases of bound optimization. In particular, we show that all limits points of any 
    sequence $\{\ZZ^i, \MM^i\}_{i \in \mathbb N}$ produced by our algorithm are stationary points.
    
    \subsection{Background}

        In this section, we introduce the minimal set of required elements to understand Zangwill's theory, upon which our own convergence result is based. 
        We first introduce the central concept of a point-to-set map $\psi$, which maps a point $\thetab \in \Theta$ to a set of points $\psi(\thetab) \subset \Theta$. 
        Map $\psi$ could be understood as the iteration of an optimization algorithm that, from a point $\thetab^{i}$ in some space $\Theta$
        of the optimization variables, outputs a new point $\thetab^{i+1} \in \psi(\thetab^{i}) \subset \Theta$ from a set of local candidate solutions. 
        We further recall the important notion of closedness, which generalizes the concept of continuity in standard point-to-point maps to point-to-set maps:\\
        \begin{Defn}
            (Closedness)  Let us consider two converging sequences:
            \begin{align}
                \thetab^{n}  \xrightarrow[n \to \infty]{} \thetab^* \nonumber\\
                \widetilde{\thetab}^{n} \xrightarrow[n \to \infty]{} \widetilde{\thetab}^* \nonumber
            \end{align}
            Let us further assume that:
            \begin{align}\label{eq:closed_map}
                \forall n \in \mathbb N,~\widetilde{\thetab}^{n} \in \psi(\thetab^{n})   
            \end{align}
            Then, the point-to-set map $\psi$ is said to be closed at point $\thetab^*$ if relation \eqref{eq:closed_map} extends to the limit when $n \rightarrow \infty$, i.e:
            \begin{align}
                \tilde{\thetab}^* \in \psi(\thetab^*) \nonumber
            \end{align}
            The point-to-set map $\psi$ is said to be closed on the set $\Theta$ if it is closed at every point of $\Theta$.
        \end{Defn}
        
        We are now ready to enunciate Zangwill's theorem:
        \begin{theorem}(p.29 in \cite{zangwill1969nonlinear}) \label{theorem:zangwill}
            Consider a compact set $\Theta$, and let  
            $\psi: \Theta \rightarrow \mathcal{P}(\Theta)$ a point-to-set map, i.e., an algorithm generating a sequence $\{\thetab^i\}_{i \in \mathbb N}$ 
            from an initial point $\{\thetab^0\}$, via iteration $\thetab^{i+1} \in \psi(\thetab^i)$.    
            Let $\Gamma \subset \Theta$ denotes a given solution set, e.g., the set of fixed points of the algorithm: $\Gamma = \{\thetab \in \Theta \, | \, \psi(\thetab)= \thetab \}$. 
            Assume that: \vspace{0.5em}\\
            (1) $\psi(\thetab)$ is nonempty and $\psi$ is closed at $\thetab$, $\forall \thetab \in \Theta$   \vspace{0.5em}\\
            (2) There is a continuous function $\mathcal{L}: \thetab \rightarrow \mathbb R$ such that: $\forall \thetab \in \Theta$, $\quad \forall ~ \thetab' \in \psi(\thetab)$:
            \begin{itemize}
                \item $\thetab' \notin \Gamma \implies \mathcal{L}(\thetab') < \mathcal{L}(\thetab)$
                \item $\thetab' \in \Gamma \implies \mathcal{L}(\thetab') \leq \mathcal{L}(\thetab)$ 
            \end{itemize}
            Then any sequence $\{\thetab^i\}_{i \in \mathbb N}$ defined by $\thetab^{i+1} \in \psi(\thetab^i)$ has all of its limit points in $\Gamma$.
        \end{theorem}
        
        As noted in \cite{sriperumbudur2009convergence}, the general idea to prove the convergence of an iterative algorithm is to properly set $\Gamma$ and $\mathcal{L}$. The natural choice for $\Gamma$ is to set it as the set of fixed points of the algorithm $\Gamma = \{\thetab \in \Theta \, | \, \psi(\thetab)= \thetab \}$. $\mathcal{L}$ is the objective function that the algorithm minimizes. With that in mind, assumption (2) in \autoref{theorem:zangwill} simply ensures that, even though the algorithm has not yet reached stationary points $\Gamma$, 
        objective function $\mathcal{L}$ is strictly decreasing. \autoref{theorem:zangwill} is, unfortunately, not straightforwardly applicable to our SLK algorithm, as it considers
        a single algorithm update $\psi$, while SLK alternates between both $\ZZ$- and $\MM$-updates, which handle two completely different types of optimization variables (assignments and prototypes). A corollary proposed in \cite{fiorot1979composition} enables us to verify the assumptions on each of type of updates independently while preserving the final result:
        
        \begin{corollary}(Corollary 2. in \cite{fiorot1979composition}) \label{corollary:fiorot}
            Consider $\Theta$ a compact set, $\Gamma \subset \Theta$, $J$ a finite set of indices, a set of point-to-set maps $\{\psi_j: \Theta \rightarrow \mathcal{P}(\Theta), ~j \in J\}$, and $\mathcal{L}: \thetab \rightarrow \mathbb R$ a continuous function.  Using the sequence of indices $I: \mathbb N \rightarrow J$, assume that for any $n \in \mathbb N$, for any $\thetab \in \Theta \setminus \Gamma$: \vspace{0.5em}\\
            (1) $\psi_{I(n)}(\thetab)$ is nonempty and $\psi_{I(n)}$ is closed at $\thetab$ \\
            (2) $\thetab' \in \psi_{I(n)}(\thetab) \implies \mathcal{L}(\thetab') < \mathcal{L}(\thetab)$ \vspace{0.5em} \\
            Then any sequence $\{\thetab^n\}_{n \in \mathbb N}$ defined by $\thetab^{n+1} \in \psi_{I(n)}(\thetab^n)$ has all of its limit points in $\Gamma$.
        \end{corollary}
        
        \subsection{Main result}
            In this section, we derive our convergence result for SLK-Means, but the result and proof are very similar for SLK-modes. We frame an inner iteration of the $\ZZ$-update defined by \eqref{eq:s_inner_update} and the whole $\MM$-update in \eqref{eq:prot_mean_update} as two distinct point-to-set maps:
            \begin{align}
                \psi_{\ZZ}(\ZZ, \MM) &= \argmin_{\ZZ} \mathcal{A}(\ZZ, \MM) \times \{\MM\} \\
                \psi_{\MM}(\ZZ, \MM) &= \{\ZZ\} \times \argmin_{\MM} \mathcal{R}(\ZZ, \MM)
            \end{align}
            $\psi_{\MM}$ and $\psi_{\ZZ}$ are the building blocks of our SLK method: a full $\ZZ$-update simply consists in sequentially applying $\psi_{\ZZ}$ until inner convergence. On the other hand, the $\MM$-update is captured by a single application of $\psi_{\MM}$.
            Finally, we define $\Gamma$ as the intersection of the sets of fixed points of $\psi_\ZZ$ and $\psi_\MM$:
            \begin{align}
                \Gamma = \{(\ZZ, \MM) ~:~ (\ZZ,\MM) = \psi_\ZZ(\ZZ, \MM) = \psi_\MM(\ZZ, \MM) \}
            \end{align}
            
            \begin{Prop}
                Starting from some initial point $\{\ZZ^{0}, \MM^{0}\}$ such that $\mathcal{R}(\ZZ^{0}, \MM^{0}) < \infty$, every limit point of the sequence $\{\ZZ^{i}, \MM^{i}\}$ built using the SLK-Means \autoref{algo:slk} is in $\Gamma$.
            \end{Prop}
            
            \begin{proof} 
                The idea of the proof is to apply Corollary \ref{corollary:fiorot} with the right ingredients. First, we define the set of point-to-set maps as $\{\psi_{\MM}, \psi_{\ZZ}\}$, and the index function $I$ as the one given by the SLK method. We now define our optimization-variable space, and show its compactness. Given that the objective keeps decreasing, $(\ZZ^{i}, \MM^{i}),~i \in \mathbb N$ has to live in the following space:
                \begin{align}
                    \Theta =  L_0 \times \Delta_K
                \end{align}
                where $L_0=\{\MM \in \mathbb R^{K\times d} \, | \, \mathcal{E}(\ZZ^0, \MM) \leq \mathcal{E}(\ZZ^0, \MM^0)\}$ represents a sublevel set of $\mathcal{E}(\ZZ^0, .)$. \\

                \textbf{Compactness of $\Theta$:} Provided the continuity of $\mathcal{E}$, $L_0$ is closed as the inverse image of a continuous function. Furthermore, it is bounded as $\mathcal{E}(\ZZ^0, \MM) \xrightarrow[\norm{\MM} \to \infty]{} \infty$. Therefore, $L_0$ is compact. Furthermore, simplex $\Delta_K$ is a compact subset of $\mathbb{R}^K$. Hence, as the product of two compact sets, $\Theta$ is compact. \\
                    
                \textbf{Assumption (1): } To prove the closedness assumption (1), we will use the following lemma from \cite{gunawardana2005convergence}:
                    \begin{lemma} \label{lemma:closedness} (Proposition 7 of Appendix A in \cite{gunawardana2005convergence})
                        For a continuous function $\phi: A \times B \rightarrow B$ defined as
                        \begin{align}
                            \psi(\abf) = \argmin_{\bb \in B} \phi(\abf, \bb)
                        \end{align}
                        Then $\psi$ is closed at $\abf$ if $\psi(\abf)$ is nonempty.
                    \end{lemma}
                    Using the continuity of both $\mathcal{A}$ and $\mathcal{R}$ in both $\ZZ$ and $\MM$, along with Lemma \ref{lemma:closedness}, the closedness of 
                    $\psi_\ZZ$ and $\psi_\MM$ follows. \\
                    
                \textbf{Assumption (2):} We now prove that the objective function strictly decreases for non-stationnary points. The argument relies on the strict convexity of $\mathcal{A}(\ZZ, \MM)$ in $\ZZ$ and of $\mathcal{R}(\ZZ, \MM)$ in $\ZZ$ and $\MM$. Specifically, consider $n \in \mathbb N$ and $\{\ZZ, \MM\} \notin \Gamma$; $\exists~\psi \in \{\psi_{\ZZ}, \psi_{\MM}\} \, | \, \{\ZZ, \MM\} \neq \psi(\ZZ, \MM)$. SLK pre-defined sequence ensures that such $\psi$ is chosen at iteration $n$. Without loss of generality, assume $\psi=\psi_{I(n)}=\psi_{\ZZ}$. Then, recall that $\mathcal{A}(\ZZ, \MM)$ is strictly convex in $\ZZ$, which implies that it has a unique minimum. Therefore, either the current solution $\ZZ^i$ is already a stationnary point, with $\ZZ^n = \argmin_\ZZ \mathcal{A}(\ZZ^n, \MM^n)$, which implies $\{\ZZ^n, \MM^n\} \in \Gamma$ (impossible), or the current solution is not a stationary point. In this case, for $\ZZ^{n+1} \in \psi_{\ZZ}(\ZZ^n, \MM^n)$, we inevitably have $\mathcal{A}(\ZZ^{n+1}, \MM^n) < \mathcal{A}(\ZZ^n, \MM^n)$.
                        
            \end{proof}

\begin{table*}[t]
\caption{Average accuracy (in \%) for the few-shot learning experiments on the \textit{mini}ImageNet, \textit{tiered}ImageNet and CUB dataset. The best results are highlighted in bold font. The shaded areas are the results based on clustering approaches for transductive few-shot inference.}
\label{tab:mini-tiered}
\begin{center}
\begin{tabular}{llcccccccc}
\toprule
& & & &\multicolumn{2}{c}{\textbf{\textit{mini}ImageNet}}&\multicolumn{2}{c}{\textbf{\textit{tiered}ImageNet}}&\multicolumn{2}{c}{\textbf{CUB}} \\
\textbf{Methods} & \textbf{Reference} & \textbf{Transductive} & \textbf{Network}&\textbf{1-shot}& \textbf{5-shot}& \textbf{1-shot}& \textbf{5-shot}& \textbf{1-shot}& \textbf{5-shot}\\ 
\midrule
Baseline & ICLR 2019\cite{chen2018a} &\multirow{10}{*}{\xmark} & ResNet-18 & 51.87  & 75.68  &-&-& 67.02&83.58\\
RelationNet & CVPR 2018 \cite{sung2018learning} & & ResNet-18 & 52.48  & 69.83  &-&-&68.58 &84.05\\
MatchingNet & NeurIPS 2016 \cite{Vinyals2016MatchingNF} & & ResNet-18 & 52.91  & 68.88  &-&-&73.49 &84.45\\
ProtoNet & NeurIPS 2017 \cite{snell2017prototypical} & & ResNet-18 & 54.16  & 73.68  &-&-& 72.99 & 86.64\\
MTL & CVPR 2019 \cite{sun2019mtl}  & & ResNet-12 & 61.20  & 75.50  &-&-&-&-\\
vFSL & ICCV 2019 \cite{Variationalfewshot}  & & ResNet-12 & 61.23  & 77.69  &-&-&-&-\\
MetaoptNet& CVPR 2019 \cite{lee2019meta}  & & ResNet-12 & 62.64  & 78.63  & 65.99  & 81.56& & \\
Neg-cosine & ECCV 2020 \cite{liu2020negative}& & ResNet-18 & 62.33  & 80.94  & 69.68  & 84.56&-&- \\
SimpleShot & arXiv 2019 \cite{wang2019simpleshot}& & ResNet-18 & 63.10  & 79.92  & 69.68  & 84.56& 70.28&86.37 \\
Distil  & ECCV 2020 \cite{tian2020rethinking}& & ResNet-18 & 64.82  & 82.14  & 71.52  & 86.03&-&- \\
\hline
MAML & ICML 2017 \cite{Finn2017ModelAgnosticMF} & & ResNet-18 & 49.61  & 65.72  &-&-& 68.42& 83.47 \\
RelationNet + T & NeurIPS 2019 \cite{can} & & ResNet-12 & 52.40  & 65.36  &-&-&-&-\\
ProtoNet + T & NeurIPS 2019 \cite{can} & & ResNet-12 & 55.15  & 71.12  &-&-&-&-\\
MatchingNet + T & NeurIPS 2019 \cite{can} & & ResNet-12 & 56.31  & 69.80  &-&-&-&-\\
Team & ICCV 2019 \cite{team}& & ResNet-18 & 60.07 & 75.90 & - & -&-&-\\
TPN & ICLR 2019 \cite{liu2018learning} & & ResNet-12 & 59.46 & 75.64 & - & -&-&-\\
Entropy-min & ICLR 2020 \cite{Dhillon2020A} & & ResNet-12 & 62.35  & 74.53  &-&-&-&-\\
DPGN & CVPR 2020 \cite{yang2020dpgn} & & ResNet-18 & 66.63  & 84.07  & 70.46 & 86.44&-&-\\
CAN+T & NeurIPS 2019 \cite{can}& & ResNet-12 & 67.19 & 80.64 & 73.21 & 84.93&-&-\\
LaplacianShot & ICML 2020 \cite{Laplacian}  & & ResNet-18 &70.74  &82.33 & 77.60 & 86.25&79.93 &88.59 \\
\rowcolor{Gray}
K-means & & & ResNet-18 &71.11  &81.87   &77.94  &85.99&80.30 &88.51 \\
\rowcolor{Gray}
K-modes & & & ResNet-18 &72.31 &82.21  &78.85  &86.12& 81.73&88.58 \\
\rowcolor{Gray}
\rowcolor{Gray}
SLK-MEANS &  & & ResNet-18&72.58  & \textbf{82.92}  &79.22  &86.50 &81.40 &\textbf{88.61} \\
\rowcolor{Gray}
SLK-MS &  & \multirow{-13}{*}{\cmark} & ResNet-18 &\textbf{73.10}  &82.82   & \textbf{79.99} & \textbf{86.55}&\textbf{81.88} &88.55\\
\hline
LEO & ICLR 2019 \cite{rusu2018metalearning} & \multirow{6}{*}{\xmark}& WRN & 61.76  &77.59  & 66.33  & 81.44&-&- \\
ProtoNet & NeurIPS 2017 \cite{snell2017prototypical} & & WRN & 62.60  & 79.97  &-&-&-&-\\
MatchingNet & NeurIPS 2016 \cite{Vinyals2016MatchingNF} & & WRN & 64.03  & 76.32  &-&-&-&-\\
CC+rot & ICCV 2019 \cite{gidaris2019boosting}  & & WRN & 62.93  & 79.87  & 70.53  & 84.98&-&- \\
FEAT & CVPR 2020 \cite{ye2020fewshot}  & & WRN & 65.10  & 81.11  & 70.41  & 84.38&- &- \\
SimpleShot & arXiv 2019 \cite{wang2019simpleshot}& & WRN & 65.87 & 82.09  & 70.90  & 85.76&-&- \\
\hline
AWGIM & CVPR 2020 \cite{guo2020attentive}  & & WRN & 63.12  & 78.40  & 67.69  & 82.82&-&- \\ 
Entropy-min & ICLR 2020 \cite{Dhillon2020A} & & WRN & 65.73  & 78.40  & 73.34  & 85.50&-&- \\
SIB & ICLR 2020 \cite{hu2020empirical}& & WRN & 70.0 & 79.2  & 70.90  & 85.76&-&- \\
BD-CSPN & ECCV 2020 \cite{liu2019prototype} & & WRN & 70.31  & 81.89  & 78.74  & 86.92&-&- \\
SIB+E$^3$BM & ECCV 2020 \cite{liu2020ensemble} & & WRN & 71.4  & 81.2  & 75.6  & 84.3&-&- \\
LaplacianShot & ICML 2020 \cite{Laplacian}  & & WRN &73.44  & 83.93  & 78.80 & \textbf{87.72}&-&- \\ 
\rowcolor{Gray}
K-means & & & WRN &73.80  &\textbf{84.62}   &79.78  & 87.23&-&-\\
\rowcolor{Gray}
K-modes & & & WRN &74.78  &84.45   & 80.67 & 87.23&-&-\\
\rowcolor{Gray}
\rowcolor{Gray}
SLK-MEANS &  & & WRN&74.75  &84.61   &80.55  &87.57&-&- \\
\rowcolor{Gray}
SLK-MS &  & \multirow{-10}{*}{\cmark} & WRN &\textbf{75.17}  &84.28   &\textbf{81.13}  & 87.69&-&-\\
\bottomrule
\end{tabular}
\end{center}
\end{table*}

\section{Few-shot learning experiments}

    In the following, we describe our experimental setup and report the results of various clustering approaches in the few-shot learning scenario.
    
    \subsection{Datasets}
    For this task, we resort to three widely used benchmarks for few-shot classification: \textit{mini}ImageNet, \textit{tiered}ImageNet and CUB-200-2011. The details 
    of these standard few-shot benchmarks are given below.
    
    \textbf{\textit{mini}ImageNet} is a subset of the larger ILSVRC-12 dataset \cite{ILSVRC15}. It has a total of 60,000 color images with 100 classes, where each class 
    has 600 images of size $84 \times 84$, following  \cite{Vinyals2016MatchingNF}. We use the standard split of 64 classes for base training, 16 for validation and 20 for 
    testing \cite{Ravi2017OptimizationAA,wang2019simpleshot}. 
    
    \textbf{\textit{tiered}ImageNet} \cite{ren18fewshotssl} is also a subset of the ILSVRC-12 dataset, but with 608 classes instead. Following \cite{wang2019simpleshot}, we 
    split the dataset into 351 classes for base training, 97 for validation and 160 for testing. As in the \textit{mini}ImageNet dataset, the images are 
    resized to $84 \times 84$ pixels.
    
    \textbf{CUB-200-2011} \cite{wah2011caltech} is a fine-grained image classification dataset with 200 categories. Following the setting in \cite{chen2018a} for few-shot 
    classification, we split the CUB dataset into 100 classes for base training, 50 for validation and 50 for testing. Similarly to the previous datasets, the images are 
    resized to $84 \times 84$ pixels.
    
    \subsection{Evaluation Protocol}
    We follow standard evaluation practices in few-shot learning, using 600 tasks (5-way classification) randomly sampled from the test classes. 
    The query set of each task contains 15 images per class, and the average accuracy over 600 few-shot tasks is reported, for both the 1-shot and 5-shot supervision scenarios. 
    
    \subsection{Architectures}
    We evaluate all clustering methods on two different network models as feature extractor function $\ftheta$, following the same settings used recently in \cite{Laplacian}. 
    Particularly, we consider the following standard models for the few-shot learning experiments:
    
    \textbf{ResNet-18} is based on the standard deep residual network, ResNet \cite{he2016deep}, whose first two down-sampling layers are removed, by setting the stride to 1 in the first convolutional layer and removing the first max-pooling layer. Furthermore, the kernel size of the first convolutional layer is set to $3 \times 3$, instead of $7 \times 7$ in the original model. The used ResNet-18 model has 8 basic residual blocks, and the dimension of the extracted features is 512.
    
    \textbf{WRN} \cite{WRN} widens the residual blocks by adding more convolutional layers and feature planes. In our particular case, we used an architecture with 28 convolutional layers, along with a widening factor of 10. This network yields an extracted feature vector of dimension 640.
    
    \subsection{Implementation Details}
    
        \textbf{Base training: } We utilize the publicly available trained models from \cite{Laplacian}\footnote{https://github.com/imtiazziko/LaplacianShot}. The network models are 
        trained using the standard cross-entropy loss on the training set of $A$ base classes, without resorting to meta-learning and episodic-training strategies. Training starts 
        from scratch and runs during 90 epochs. The models are optimized using SGD. Similarly to \cite{Laplacian,wang2019simpleshot}, the initial learning rate is set to 0.1 and scaled by 0.1 at the 
        45th and 66th epoch for \textit{mini}ImageNet and CUB, and after every 30 epochs for \textit{tiered}ImageNet. The mini-batch size is set to 256 for ResNet-18 and 128 for 
        WRN.
        
        \textbf{Inference procedure: }Before running any clustering approach in the few-shot learning setting, we extract the features from both the support and query samples with the initial $\ftheta$ trained on the base classes. Similarly to \cite{wang2019simpleshot,Laplacian}, we perform CL2 normalization on the features. This amounts to computing the mean of 
        the base-class features, followed by centering the extracted support and query features via subtracting the mean. Then, we perform an L2 normalization on the centered features. 
        Similarly to \cite{liu2019prototype,Laplacian}, we also perform a bias correction by computing the cross-class bias from the difference between the mean features of the support set $\XXs$ and the query set $\XXq$. Then, we add this bias to each query features vector before performing the clustering for a given few-shot task. 
        For each task, the initial prototypes $\{\mm_k^0\}_{k=1}^K$ for the clustering algorithms (see Alg. \ref{algo:slk}) are either the support examples $\xx_p \in \XXs^k$ of each class in the 1-shot setting, or the mean of the support examples within each class in the 5-shot setting. Once the clustering of a task is performed, we label each query sample according to the class label of the support examples belonging to the same cluster.
        
        \textbf{SLK hyperparameters: } We define $\mathcal{N}_p^{\rho}$ as the set of the $\rho$ nearest neighbors of data point $\xx_p$. Mode estimation is based on the Gaussian kernel as per \eqref{eq:rbf_kernel}, with $\sigma^2$ estimated as:
        \begin{align}
            \sigma^2 = \frac{1}{N\rho}\sum_{\mathbf{x}_p\in \mathbf{X}} \sum_{\mathbf{x}_q\in \mathcal{N}_p^{\rho}}\|\mathbf{x}_p-\mathbf{x}_q\|^2
        \end{align}
        The Laplacian affinities $w_\mathcal{L}$ are built using the $\rho$-nearest neighbors:
        \begin{align}
            w_{\mathcal{L}}(\xx_p,\xx_q) = \begin{cases}
            1 & \text{if } \mathbf{x}_q\in \mathcal{N}_p^{\rho} \\
            0 & \text{otherwise}
            \end{cases}
        \end{align}
        This yields a sparse affinity matrix, which is efficient in terms of memory and computational complexity.
        In our experiments, $\rho$ is simply chosen from three typical values (3, 5 or 10) tuned over training-data by sampling 500 few-shot tasks from the training set. Such procedure yielded a fixed $\rho=3$ during inference. Since no test data is involved in choosing the value of $\rho$, this procedure is not biased towards the testing classes. Then, we tune parameter 
        $\lambda$ on the validation classes by sampling $500$ few-shot tasks and evaluating a range of $\lambda$ values from $[0.1, 0.3, 0.5, 0.7, 0.8, 1.0]$. The best $\lambda$ values corresponding to the best average 1-shot and 5-shot accuracy over the 500 tasks on the validation classes are selected for inference on the test classes.
        \begin{table}
            \caption{ Average inference time (in seconds) for the 5-shot tasks. The results are for \textit{mini}ImageNet dataset with WRN network.}
            \label{tab:timing}
            \begin{center}
            \begin{small}
            \begin{tabular}{lc}
            \toprule
            \textbf{Methods} & \textbf{inference time} \\
            \midrule
            SimpleShot \cite{wang2019simpleshot}& 0.009 \\
            Transductive tuning \cite{Dhillon2020A} & 20.7 \\
            \rowcolor{Gray}
            K-means & 1.17\\
            \rowcolor{Gray}
            K-modes & 1.10\\
            \rowcolor{Gray}
            \rowcolor{Gray}
            SLK-MEANS &1.71\\
            \rowcolor{Gray}
            SLK-MS &1.15 \\
            \bottomrule
            \end{tabular}
            \end{small}
            \end{center}
        \end{table}

    \subsection{Results}
    The results on \textit{mini}Imagenet, \textit{tiered}Imagenet and CUB datasets are reported in Table \ref{tab:mini-tiered}. It is very interesting to observe that the clustering-based approaches (shaded gray in the tables) consistently outperform transductive few-shot learning methods, yielding state-of-the-art on these standard benchmarks, without resorting to complex meta-learning and episodic-training strategies. In the 1-shot settings, our SLK-MS clustering outperforms the best-performing state-of-the-art few-shot method by important 
    margins (2 to $3\%$), across data sets and models. It also outperforms several convoluted meta-learning methods by substantial margins, up to $10\%$ or more, including many very recent tranductive methods. Surprisingly, our results show that even the standard K-means clustering procedure already achieves competitive performances in 
    comparison to the state-of-the-art. The standard K-means procedure is on par with the best-performing method in the tranductive setting, and outperforms substantially several very recent 
    methods based on complex meta-learning and/or episodic training schemes. These surprising results point to the limitations of the current few-shot benchmarks, and question the viability of a large body of convoluted few-shot learning techniques in the recent literature. It also suggests rethinking transductive few-shot learning and using simple constrained clustering 
    mechanisms as baselines.
    
    \subsection{Inference time}
    Table \ref{tab:timing} reports an evaluation in terms of efficiency. This is done by measuring the average time required for inference on the 5-shot tasks in \textit{mini}ImageNet. 
    Note that all the methods listed in Table \ref{tab:timing} are transductive, except SimpleShot. The transductive network fine-tuning approach in \cite{Dhillon2020A}, for instance, 
    is significantly slower than the other methods. This is due to the gradient updates performed by \cite{Dhillon2020A} at inference over all the network parameters. 
    While the clustering methods are more efficient than the approach in \cite{Dhillon2020A}, they require longer inference times than inductive predictions \cite{wang2019simpleshot}. 

\section{Experiments on clustering}
    We conducted a comprehensive experimental analysis of our SLK approaches on 5 different data sets, each having different features. To evaluate the clustering performance, we 
    use two metrics widely adopted in the clustering community: the Normalized Mutual Information (NMI) \cite{strehl2002cluster} and the Clustering Accuracy (ACC). 
    The optimal mapping of clustering assignments to the true labels is determined by using the Kuhn-Munkres algorithm \cite{munkres1957algorithms}. 
    We further report evaluations and comparisons in terms of optimization quality, which we measure by the values obtained at convergence for the discrete-variable objective ${\cal E}$ in \eqref{eq:lkmode0}.
    \begin{table*}[t]
    \caption{Clustering results as NMI/ACC.}
    \label{tab:tab-results-imc}
    \centering
    \begin{tabular}{lccccccc}
    \toprule
    Algorithm&MNIST&MNIST (GAN)&LabelMe (Alexnet) & LabelMe (GIST)& YTF & Shuttle&Reuters\\
    \midrule
    K-means& 0.53/0.55&0.68/0.75&0.81/0.90&0.57/0.69 &0.77/0.58&0.22/0.41&0.48/0.73\\
    K-modes & 0.56/0.60&0.69/0.80&0.81/0.91&0.58/0.68 &0.79/0.62&0.33/0.47&0.48/0.72\\ 
    NCUT & 0.74/0.61&0.77/0.67&0.81/\textbf{0.91}&0.58/0.61 &0.74/0.54&\textbf{0.47}/0.46&-\\
    KK-means&0.53/0.55&0.69/0.68&0.81/0.90&0.57/0.63 &0.71/0.50& 0.26/0.40&-\\
    LK&-&-&0.81/\textbf{0.91}&0.59/0.61 &0.77/0.59 &-&-\\
    Spectralnet&-&-&-&-&-&-&0.46/0.65 \\
    SLK-MS&\textbf{0.80}/\textbf{0.79}&0.86/\textbf{0.94}&\textbf{0.83}/\textbf{0.91}&\textbf{0.61}/\textbf{0.72} &\textbf{0.82}/\textbf{0.65}&0.45/0.70&0.43/\textbf{0.74}\\
    SLK-Means&0.78/\textbf{0.75}&0.84/0.91&\textbf{0.83}/\textbf{0.91}&0.60/0.68 &\textbf{0.82}/0.64&0.31/\textbf{0.71}&\textbf{0.56}/\textbf{0.77}\\
    \bottomrule
    \end{tabular}
    \end{table*}
    \begin{figure*}
    \centering
    \subfloat[MNIST (small)]{\includegraphics[width=.27\textwidth,height=.20\textwidth]{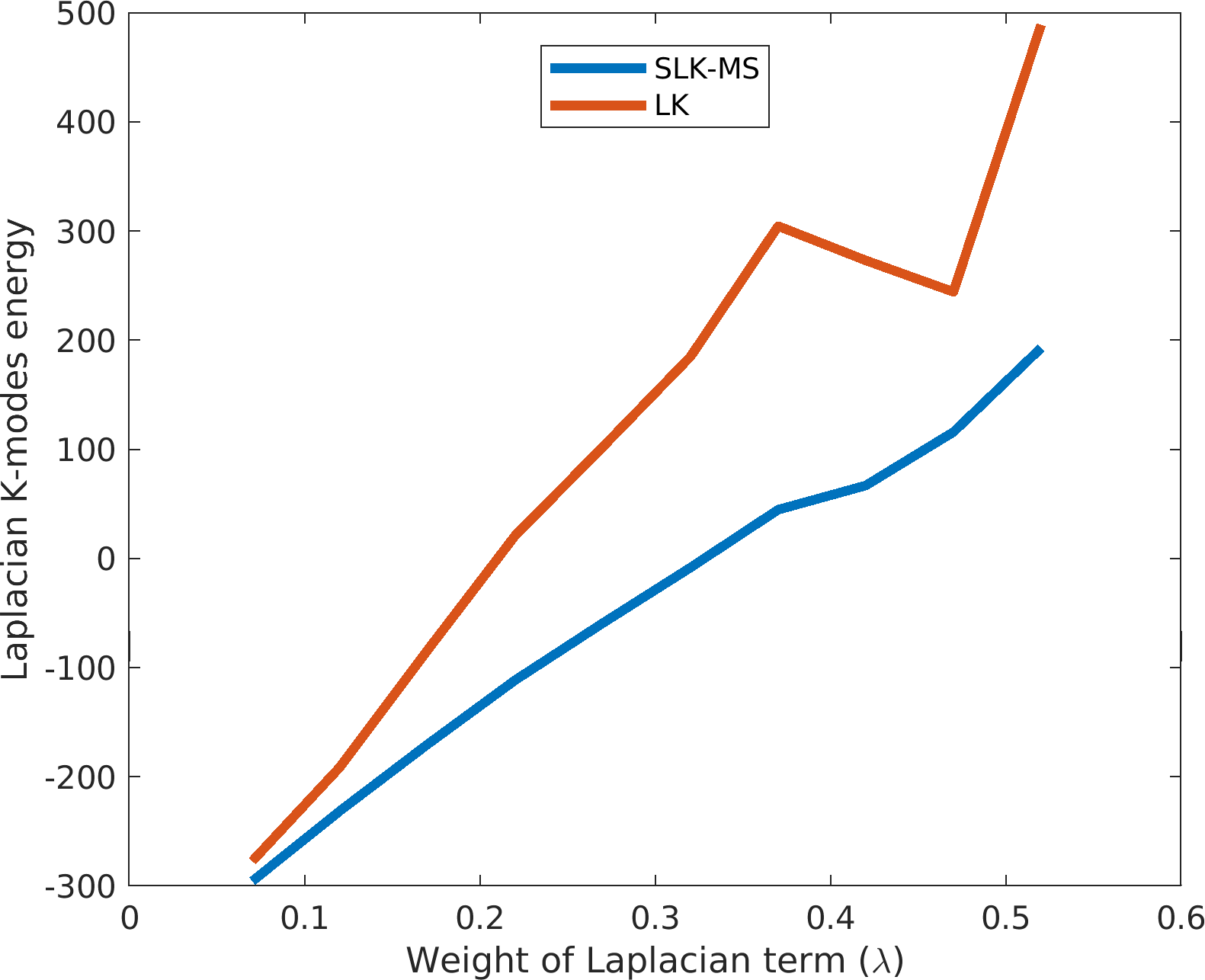}}
    \subfloat[LabelMe (Alexnet)]{\includegraphics[width=.27\textwidth,height=.20\textwidth]{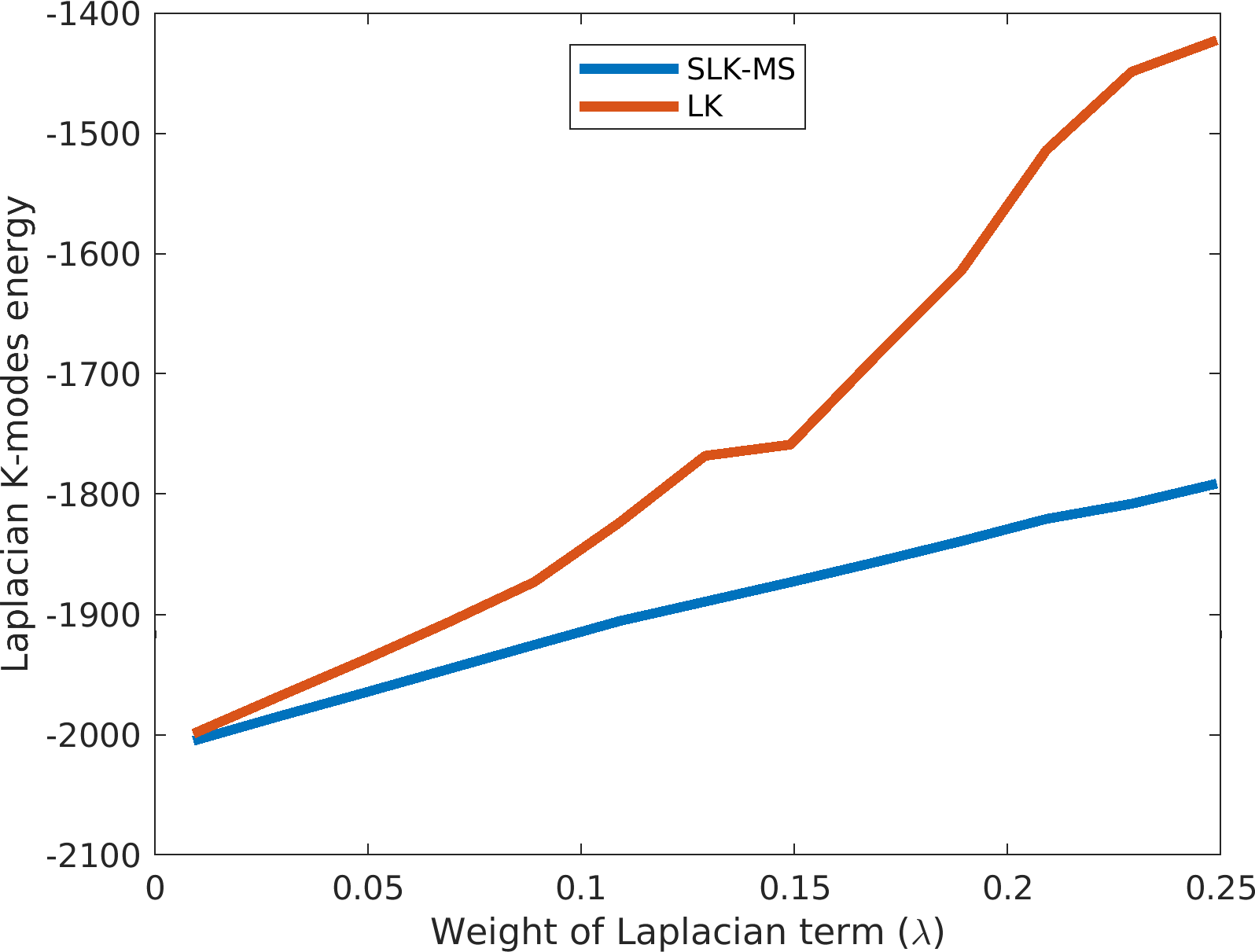}}
    \caption{Discrete-variable objective \eqref{eq:lkmode0}: Comparison of the objectives obtained at convergence for SLK-MS (ours) and LK \cite{WangCarreira-Perpinan2014}. 
    The objectives at convergence are plotted versus different values of parameter $\lambda$.}
    \label{fig:energy}
    \end{figure*}
    
    We resort to well-known benchmarks to report comprehensive evaluations of the proposed algorithm\footnote{Code is available at: https://github.com/imtiazziko/SLK}, as well as comparisons to the following related baseline methods: Laplacian K-modes (LK) \cite{WangCarreira-Perpinan2014}, K-means, NCUT \cite{ShiMalik2000}, K-modes \cite{salah2014convex,carreira2015review}, Kernel K-means (KK-means) \cite{dhillon2004kernel,Tang2019KernelCK} and Spectralnet \cite{shaham2018spectralnet}. 
    
    \subsection{Clustering datasets} 
       \begin{table*}[t]
        \caption{Datasets used in the clustering experiments.}
        \label{tab:tab-results-data}
        \centering
        \begin{tabular}{l*{4}{d{0}}}
        \toprule
        Datasets & \text{Samples}~(N) & \text{Dimensions}~(D) & \text{Clusters}~(L) &\text{Imbalance} \\
        \midrule
        MNIST (small) &2,000&784&10&1\\
        MNIST &70,000&784&10&\sim 1\\
        MNIST (GAN) &70,000&256&10& \sim 1\\
        Shuttle &58,000&9&7&4,558\\
        LabelMe (Alexnet)&2,688&4,096&8&1\\
        LabelMe (GIST) &2,688 &44,604 &8 & 1\\
        YTF &10,036&9,075&40& 13\\
        Reuters (code) &685,071&10&4& \sim 5\\
        \bottomrule
        \end{tabular}
    \end{table*}
        We used a mix of image datasets (MNIST, LabelMe, YTF) and non-image datasets (Reuters, Shuttle) for the clustering task. The overall summary of the datasets is given in Table \ref{tab:tab-results-data}. For each dataset, imbalance is defined as the ratio between the size of the largest and smallest clusters. We use three versions of the MNIST dataset \cite{lecun1998gradient}, which are detailed below. MNIST is the original dataset and contains the whole set of $70,000$ images. In order to compare to LK   \cite{WangCarreira-Perpinan2014}, which does not scale up for large datasets, we use MNIST (small). This reduced version includes only $2,000$ images by randomly sampling $200$ images per class. To obtain the MNIST (GAN) dataset, we train a GAN following \cite{goodfellow2014generative} on $60,000$ training images and extract the $256$-dimensional features from the discriminator network for the whole set of $70,000$ images. We use the publicly available auto-encoder in \cite{jiangZTTZ16} to extract $10$-dimensional features to generate the Reuters (code) dataset. LabelMe \cite{oliva2001modeling} consists of $2,688$ images divided into $8$ categories. Furthermore, we used the pre-trained AlexNet \cite{krizhevsky2012imagenet} and extracted the 4096-dimensional features from the fully-connected layer, leading to the LabelMe (AlexNet) dataset. To show the performances on high-dimensional data, we extract $44604$-dimensional GIST features \cite{oliva2001modeling} for the LabelMe dataset, which is referred to as LabelMe(GIST). Finally, the Youtube Faces (YTF) \cite{wolf2011face} dataset consists of videos of faces with $40$ different subjects.

    \subsection{Implementation details}
         Similarly to the few-shot experiments, we used $\rho$-nearest neighbor affinities to build $w_\mathcal{L}$, with fixed $\rho = 5$ for all the datasets. For the 
         large datasets, such as MNIST, Shuttle and Reuters, we used the \emph{Flann} library \cite{flann_pami_2014} with the KD-tree algorithm, which finds approximate 
         nearest neighbors. For the other smaller datasets, we used an efficient implementation of exact nearest-neighbor computations. We used the same sparse $\KK_{\mathcal{L}}$ matrix for the pairwise-affinity algorithms we compared with, i.e., NCUT, KK-means and Laplacian K-modes. Furthermore, for each of these baseline methods, we evaluated the default setting of affinity construction with tuned $\sigma$, and report the best result found. 
        Initial centers $\{\mm^0_k\}_{k=1}^{K}$ are based on K-means++ seeds \cite{arthur2007k}. We choose the best initial seed and regularization parameter $\lambda$ empirically based on the accuracy over a validation set (10\% of the total data). The $\lambda$ is determined by tuning over a small range from $1$ to $4$. In Algorithm \ref{algo:slk}, all assignment variables $\zz_p$ are updated in parallel. We run the publicly released codes for K-means \cite{scikit-learn}, NCUT \cite{ShiMalik2000}, Laplacian K-modes \cite{Carreira-Perpinan2007}, KK-means and Spectralnet \cite{shaham2018spectralnet}.
    
    \subsection{Clustering results}
        Table~\ref{tab:tab-results-imc} reports the clustering results, showing that, in most of the cases, our algorithms SLK-MS and SLK-Means yielded the best NMI and ACC values. For MNIST with the raw intensities as features, the proposed SLK-MS achieved almost ~80\% NMI and ACC, outperforming all the standard clustering methods by substantial margins. With better learned GAN features for MNIST (GAN), the accuracy (ACC) increases up to ~95\%, also outperforming standard clustering methods by substantial margins. For the Reuters (code) datasets, we used the same features and Euclidean distance based affinity as Spectralnet \cite{shaham2018spectralnet}, and obtained better NMI/ACC performances. The Shuttle dataset is quite imbalanced and, therefore, all the baseline clustering methods fail to achieve high accuracy. Notice that, in regard to ACC for the Shuttle dataset, our SLK methods outperformed all the other methods by a large margin (over $20\%$ of difference).
    
    \subsection{Comparison in terms of optimization quality}\label{sec:optimization_quality}
        To assess the optimization quality of our optimizer, we computed the values of discrete-variable objective ${\cal E}$ in model \eqref{eq:lkmode0} at convergence for our concave-convex relaxation (SLK-MS) as well as for the convex relaxation in \cite{WangCarreira-Perpinan2014} (LK). We compare the discrete-variable objectives for different values of $\lambda$. For a fair comparison, we use the same initialization, $\sigma$, $w_{\mathcal{F}}(\xx_p,\xx_q)$, $\lambda$ and mean-shift (MS) modes for both methods.
        As shown in the plots in Figure \ref{fig:energy}, and reported in Table \ref{tab:tab-results-energy}, our relaxation consistently obtained lower values of discrete-variable objective ${\cal E}$ at convergence than the convex relaxation in \cite{WangCarreira-Perpinan2014}. 
        
        \begin{table}
            \caption{Discrete-variable objectives at convergence for LK \cite{WangCarreira-Perpinan2014} and SLK-MS.}
            \label{tab:tab-results-energy}
            \centering
            \begin{tabular}{lSS}
            \toprule
            \text{Datasets} & \text{LK} \cite{WangCarreira-Perpinan2014}& \text{SLK-MS (ours)}\\
            \midrule
            MNIST (small) & \num{273.25}  &\num{67.09}  \\
            LabelMe (Alexnet) &\num{-1.51384e+03} &\num{-1.80777e+03} \\
            LabelMe (GIST) &\num{ -1.95490e+03} & \num{-2.02410e+03} \\
            YTF &\num{-1.00032e+4} &\num{-1.00035e+4} \\
            \bottomrule
            \end{tabular}
        \end{table}

\section{Conclusion}

We presented a general formulation for clustering and transductive few-shot learning, integrating Laplacian regularization, prototype-based terms and 
supervision constraints. Based on a concave-convex relaxation, our bound optimizer performs parallel point-wise updates, while guaranteeing convergence.
We reported comprehensive clustering and few-shot experiments over various data sets, showing that our methods yield competitive performances, in regard 
to both accuracy and optimization quality. Using standard training on the base classes, and without resorting to complex meta-learning and episodic-training 
strategies, our clustering methods yield state-of-the-art results in the transductive few-shot setting, outperforming significantly a large number of convoluted 
methods in the recent literature, across various models, settings and data sets. In the 1-shot settings, our SLK-MS clustering outperforms the best-performing 
state-of-the-art few-shot method by up to $3\%$. It also outperforms several convoluted meta-learning methods by substantial margins, up to $10\%$ or more. 
Surprisingly, we found that even the basic K-means already achieves competitive performances in comparison to the state-of-the-art in few-shot learning, 
outperforming substantially several very recent methods. These surprising performances point to the limitations of the current few-shot benchmarks, question the 
viability of a large body of complex few-shot learning works in the literature, and suggest simple clustering mechanisms as baselines.

\section*{Acknowledgements}
This research was supported by the National Science and Engineering Research Council of Canada (NSERC), via its Discovery Grant program (grant CRSNG RGPIN 2019-05954). 
It was also supported by computational resources provided by Compute Canada and Calcul Quebec.

\small
\bibliographystyle{abbrv}
\bibliography{readings}

\clearpage
\appendices

\section{Proof of Proposition 1}
\label{Appendix-A}

    \textbf{Proposition 1} states that, given current solution $\ZZ^{i,n} = [s_{p,k}^{i,n}]$ at outer iteration $i$, we have the following auxiliary function (up to an additive constant) for concave-convex relaxation \eqref{Concave-Convex-relaxation} and psd affinity matrix $\KK_{\mathcal{L}}$:
    \begin{equation}
    \label{Aux-function-a}
    {\cal A}^{i, n}(\ZZ) =  \sum_{p =1}^{N} \zz_p^t (\log (\zz_p) - \abf_p^i - \lambda{\bb}_p^{i, n})
    \end{equation}
    where ${\abf}_p^i$ and ${\bb}_p^i$ are the following $K$ dimensional vectors:
    \begin{subequations} 
    \begin{align}
    {\abf}_p^i &=  [a_{p,1}^i, \dots,a_{p,K}^i]^t,  \, \mbox{\em with} \, \, a_{p,k}^i = w_{\mathcal{F}}(\xx_p,\mm_k^i) \label{general-second-2} \\
    {\bb}_p^i &= [b_{p,1}^i, \dots,b_{p,K}^i]^t,  \, \mbox{\em with} \, \, b_{p,k}^i =   \sum_{q} w_{\mathcal{L}}(\xx_p, \xx_q) s_{q,k}^i \label{general-third-2} 
    \end{align}
    \end{subequations}
    
    \begin{proof}
    
    We want to study the sign of the following difference. Discarding the constant $2\sum_p d_p$ in \eqref{tight-relaxation} for the Laplacian term $\Lcal(\ZZ)$, one can write:
    \begin{align}\label{eq:difference}
        \mathcal{R}^i(\ZZ) - {\cal A}^{i, n}(\ZZ) &= \lambda~\sum_{p} \zz_p^t (\bb_p^{i, n} - \sum_q w_{\mathcal{L}}(\xx_p, \xx_q)\zz_q)  \\
        &= \lambda~\sum_{p, q} \zz_p^t w_{\mathcal{L}}(\xx_p, \xx_q) (\zz_q^i - \zz_q)
    \end{align}
    First, instead of considering the $N \times K$ matrix $\ZZ$, let us represent our assignment variables by a flattened vector $\zz \in [0,1]^{KN}$, that takes the form $\zz = [\zz_1,\zz_2,  \dots, \zz_N]$.
    Recall that each $\zz_p$ is a vector of dimension $K$ containing the probability variables of all labels for point $p$: $\zz_p = \lbrack s_{p,1}, \dots, s_{p, K}\rbrack^t$. Let $\Psi = - \KK_{\mathcal{L}} \otimes \mathbf{I}_{N}$, where $\otimes$ denotes the Kronecker product and $\mathbf{I}_{N}$ the $N \times N$ identity matrix. 
    One can rewrite \eqref{Concave-Convex-relaxation} in the following convenient form:
    \begin{equation}
    \label{Laplacian-bound}
    - \sum_{p,q} w_{\mathcal{L}}(\xx_p, \xx_q) \zz_p^t \zz_q = \zz^t\Psi \zz 
    \end{equation}
    Which allows us to write the difference \eqref{eq:difference} as:
    \begin{align}
        \mathcal{R}^i(\ZZ) - {\cal A}^{i, n}(\ZZ) = \lambda (\zz^t\Psi \zz - \zz^t\Psi \zz^{i, n} )
    \end{align}
    Now, notice that the Kronecker product $\Psi$ is negative semi-definite when $\KK_{\mathcal{L}}$ is positive semi-definite. Hence, the function $\zz \rightarrow \zz^t\Psi \zz$ is concave and, therefore, is upper bounded by its first-order approximation at current solution $\ZZ^{i, n}$ (outer iteration $i$, inner iteration $n$). In fact, concavity arguments are standard in deriving auxiliary functions for bound-optimization algorithms \cite{Lange2000}. With this condition, we have the following auxiliary function for the Laplacian-term relaxation:
    \begin{equation}
        \label{Aux-Function-CRF-concave}
        \zz^t\Psi \zz \leq (\zz^{i, n})^t \Psi \zz^{i, n} + (\Psi \zz^{i, n})^t (\zz - \zz^{i, n}) 
    \end{equation}
    Equation \eqref{Aux-Function-CRF-concave} directly implies that:
    \begin{align}\label{eq:interm1}
        \zz^t\Psi \zz - \zz^t\Psi \zz^{i, n} \leq 0
    \end{align}
    Which implies that:
    \begin{align}\label{eq:interm2}
        \mathcal{R}^i(\ZZ) \leq {\cal A}^{i, n}(\ZZ)
    \end{align}
    Finally note that for $\zz = \zz^{i, n}$, equation \eqref{Aux-Function-CRF-concave} becomes an equality, and so do following equations \eqref{eq:interm1} and \eqref{eq:interm2}, which finalizes the proof.
    \end{proof}

    
    
    

    
    


\section{Proof of Proposition \ref{prop:prototype_update}}
\label{appendix:k_modes_prototypes}

    Let us consider outer iteration $i$, with the associated current assignements  $\ZZ_i$. For each class $k$, let us prove that $\{ \mm_k^{i, n}\}_{n \in \mathbb N}$ is a Cauchy sequence. Recall the recursive relation:
    \begin{align}
        \mm_k^{i, n} &= \frac{\sum_p s_{p,k}w_{\mathcal{F}}(\xx_p, \mm_k^{i, n})\xx_p}{\sum_p s_{p,k}w_{\mathcal{F}}(\xx_p, \mm_k^{i, n})} \\
    \end{align}
    with $w_{\mathcal{F}}(\xx_p, \mm_k^{i, n}) = \exp(-\norm{\frac{\xx_p -  \mm_k^{i, n}}{\sigma}}^2)$, for some $\sigma > 0$. To simplify notations in what follows, we define:
    \begin{align}
        h(x) &= \exp(-x) \\
        u^n &= \sum_p s_{p,k}w_{\mathcal{F}}(\xx_p, \mm_k^{i, n}) \\
        v^n &= \sum_p s_{p,k}w_{\mathcal{F}}(\xx_p, \mm_k^{i, n})\xx_p 
    \end{align}
    \textbf{Step 1: } First, let us prove that $\{u^n\}_{n \in \mathbb N}$ is a Cauchy sequence. Recall that in a metric space, a convergent sequence is necessarily a Cauchy sequence. Therefore, we only need to show that $u^n$ is convergent (i.e bounded and strictly monotonic). \\
    Notice that for $x>0$, $0 \leq h(x) \leq 1$. Therefore:
    \begin{align}
        u^n &= \sum_p s_{p,k}h(\norm{\frac{\xx_p -  \mm_k^{i, n}}{\sigma}}^2) \\
            &\leq \sum_p s_{p,k} \leq N
    \end{align}
    Therefore, $u^n$ is bounded between 0 and N. Now, let us study the consecutive differences $\Delta^n = u^{n+1} - u^{n}$:
    \begin{align}
        \Delta^n = \sum_p  s_{p,k} \left[h(\frac{\norm{\xx_p -  \mm_k^{i, n+1}}^2}{\sigma^2}) - h(\frac{\norm{\xx_p -  \mm_k^{i, n}}^2}{\sigma^2}) \right]
    \end{align}
    Because $h$  is convex, one can say that $\forall a, b \in \mathbb R$:
    \begin{align}
        h(a) - h(b) \geq h'(b)(a-b)
    \end{align}
    And because $h'(b)=-h(b)$ in our case, one ends up with:
    \begin{align}
        h(a) - h(b) \geq h(b)(b-a)
    \end{align}
    Applied with $a = \frac{\norm{\xx_p -  \mm_k^{i, n+1}}^2}{\sigma^2}$ and $b=\frac{\norm{\xx_p -  \mm_k^{i, n}}^2}{\sigma^2}$, one can obtain:
    \[
    \begin{split}
        \Delta^n \geq& \sum_p s_{p,k} w_{\mathcal{F}}(\xx_p, \mm_k^{i, n}) \left[ \frac{\norm{\xx_p -  \mm_k^{i, n}}^2}{\sigma^2} - \frac{\norm{\xx_p -  \mm_k^{i, n+1}}^2}{\sigma^2} \right] \\ 
                 =& \frac{1}{\sigma^2}\sum_p s_{p,k} w_{\mathcal{F}}(\xx_p, \mm_k^{i, n}) \left[ \norm{\mm_k^{i, n}}^2 - \norm{\mm_k^{i, n+1}}^2  \right.\\
                  &\left.  -2<\mm_k^{i, n}, \xx_p> + 2 <\mm_k^{i, n+1}, \xx_p> \right] \\
    \end{split}
    \]
    Now is time to recall recursive relation $\mm_k^{i, n+1} = \displaystyle\frac{v^n}{u^n}$. By simply  expanding, one can end up with:
    \begin{align}
        \Delta^n \geq& \frac{1}{\sigma^2}\left[ \norm{\mm_k^{i, n}}^2 u^n - \frac{\norm{v^n}^2}{u^n} - 2 <\mm_k^{i, n}, v^n> + 2 \frac{\norm{v^n}^2}{u^n}\right] \nonumber\\
                    =& \frac{1}{\sigma^2}u^n \left[ \norm{\mm_k^{i, n}}^2 -2<\mm_k^{i, n},\mm_k^{i, n+1}> + \norm{\mm_k^{i, n+1}}^2 \right] \nonumber\\
                    \label{eq:relation_mn_un}=& \frac{1}{\sigma^2}u^n \norm{\mm_k^{i, n}-\mm_k^{i, n+1}}^2
    \end{align}
    Therefore, $\Delta^n >0$, which shows that $\{u^n\}_{n \in \mathbb N}$ is strictly increasing. This concludes the proof that $u^n$ is a convergent sequence, and therefore a Cauchy one.\\
    
    \textbf{Step 2: } Now, on top of concluding the proof that $\{u^n\}_{n \in \mathbb N}$ is a Cauchy sequence, Eq. (\ref{eq:relation_mn_un}) also offers an interesting relation between $\{\Delta^n\}_{n \in \mathbb N}$ and the sequence of interest $\{\mm_k^{i, n}\}_{n \in \mathbb N}$, which we can use.
    Indeed, for any $n_0, m \in \mathbb N$, we can sum Eq. (\ref{eq:relation_mn_un}):
    \begin{align}
        \label{eq:telescope_sum}\sum_{n=n_0}^{n_0+m} \Delta^n \geq& \frac{1}{\sigma^2} \sum_{n=n_0}^{n_0+m} u^n \norm{\mm_k^{i, n}-\mm_k^{i, n+1}}^2 \\
        \label{eq:take_min_u}\geq& \frac{u_0}{\sigma^2} \sum_{n=n_0}^{n_0+m} \norm{\mm_k^{i, n}-\mm_k^{i, n+1}}^2 \\
        \geq& \label{eq:triangle_ineq} \frac{u_0}{\sigma^2} \norm{\mm_k^{i, n_0+m}-\mm_k^{i, n_0}}^2
    \end{align}
    Where Eq. (\ref{eq:take_min_u}) follows because $\{u^n\}_{n \in \mathbb N}$ is strictly increasing, and Eq. (\ref{eq:triangle_ineq}) follows from the triangle inequality. Now, the left-hand side of Eq. (\label{eq:telescope_sum}) can be reduced to $\sum_{n=n_0}^{n_0+m} \Delta^n= u^{n_0+M+1} - u^{n_0}$. But because we proved in Step 1 that $\{u^n\}_{n \in \mathbb N}$ was a Cauchy sequence, this difference is bounded by a constant. This concludes the proof that $\{\mm_k^{i, n}\}_{n \in \mathbb N}$ is itself a Cauchy sequence in the Euclidean space.\\
    
    \textbf{Step 3: } We just proved that $\{\mm_k^{i, n}\}_{n \in \mathbb N}$ was a Cauchy sequence. Therefore $\{\mm_k^{i, n}\}_{n \in \mathbb N}$ can only converge to a single value $\mm_k^{i, *}$. We now use the continuity of function $g_k^i$ to conclude that $\mm_k^{i, *}$ has to be a solution of the initial equation (\ref{eq:zero_eq}):
    \begin{align}
        \mm_k^{i, *} &= \lim_{n \rightarrow \infty} \mm_k^{i, n+1} = \lim_{n \rightarrow \infty} g_i^k(\mm_k^{i, n}) \\
        &= g_i^k(\lim_{n \rightarrow \infty} \mm_k^{i, n}) = g_i^k(\mm_k^{i, *})
    \end{align}

\end{document}